\documentclass[11pt]{article}
\usepackage{algorithm}
\usepackage{algorithmic}
\usepackage{amsfonts}
\usepackage{amsmath}
\usepackage{amssymb}
\usepackage{amsthm}
\usepackage{bbm}
\usepackage{float}
\newfloat{algorithm}{t}{lop}
\usepackage[margin=1in]{geometry}
\usepackage{graphicx}
\usepackage[utf8]{inputenc}
\usepackage{mathtools}
\usepackage[numbers]{natbib}
\usepackage[dvipsnames]{xcolor}
\usepackage{tikz}
\usepackage{thmtools}
\usepackage{thm-restate}
\usepackage{wrapfig}

\usepackage{hyperref}
\hypersetup{
    colorlinks=true,
    linkcolor=blue,
    urlcolor=blue,
    citecolor = blue
}

\usepackage[capitalise]{cleveref}
\creflabelformat{equation}{#2#1#3}

\newtheorem{definition}{Definition}
\newtheorem{lemma}{Lemma}

\newcommand{\A}{\mathcal{A}}
\newcommand{\aggtree}{\textsl{AggTree}}
\newcommand{\algo}{\textsl{JointExp}}
\newcommand{\apq}{\textsl{AppIndExp}}

\newcommand{\csmooth}{\textsl{CSmooth}}

\newcommand{\eps}{\varepsilon}

\newcommand{\lap}[1]{\mathsf{Lap}\left(#1\right)}
\newcommand{\lln}[1]{\mathsf{LLN}\left(#1\right)}

\newcommand{\num}[2]{\mathrm{count}_#2(#1)}

\renewcommand{\P}[2]{\mathbb{P}_{#1}\left[#2\right]}
\newcommand{\Po}[1]{\mathbb{P}_{#1}}

\newcommand{\pq}{\textsl{IndExp}}

\renewcommand{\S}{\mathcal{S}}

\newcommand{\X}{\mathcal{X}}
\newcommand{\Y}{\mathcal{Y}}

\newcommand{\arxiv}[1]{#1}
\newcommand{\narxiv}[1]{}

\title{Differentially Private Quantiles}
\author{Jennifer Gillenwater\thanks{Google New York, jengi@google.com}\and Matthew Joseph\thanks{Google New York, mtjoseph@google.com} \and Alex Kulesza\thanks{Google New York, kulesza@google.com}}

\begin{document}
\maketitle
\begin{abstract}
    Quantiles are often used for summarizing and understanding data. If that data is sensitive, it may be necessary to compute quantiles in a way that is differentially private, providing theoretical guarantees that the result does not reveal private information. However, when multiple quantiles are needed, existing differentially private algorithms fare poorly: they either compute quantiles individually, splitting the privacy budget, or summarize the entire distribution, wasting effort. In either case the result is reduced accuracy. In this work we propose an instance of the exponential mechanism that simultaneously estimates exactly $m$ quantiles from $n$ data points while guaranteeing differential privacy. The utility function is carefully structured to allow for an efficient implementation that returns estimates of all $m$ quantiles in time $O(mn\log(n) + m^2n)$. Experiments show that our method significantly outperforms the current state of the art on both real and synthetic data while remaining efficient enough to be practical.
\end{abstract}

\section{Introduction}
Quantiles are a widespread method for understanding real-world data, with example applications ranging from income~\cite{SKCM20} to birth weight~\cite{C01} to standardized test scores~\cite{G20}. At the same time, the individuals contributing data may require that these quantiles not reveal too much information about individual contributions. As a toy example, suppose that an individual joins a company that has exactly two salaries, and half of current employees have one salary and half have another. In this case, publishing the exact median company salary will reveal the new employee's salary.

\emph{Differential privacy}~\cite{DMNS06} offers a solution to this problem. Informally, the distribution over a differentially private algorithm's outputs must be relatively insensitive to the input of any single data contributor. Returning to the salary example, a differentially private method for computing the median company salary would have similar-looking output distributions regardless of which salary the new employee receives. The resulting uncertainty about any single contributor's data makes the algorithm ``private''.

In this work, we study differentially private estimation of user-specified quantiles $q_1, \ldots, q_m \in [0,1]$ for a one-dimensional dataset $X$ of size $n$.  The output quantile estimates consist of $m$ values, which we denote $o_1, \ldots, o_m$.  Ideally, the $o_j$ are as close to the dataset's actual quantiles as possible. For example, if $q_j = 0.5$, then our goal is to output $o_j$ close to the median of $X$.

Several algorithms for computing a single differentially private quantile exist (see \cref{sec:experiments}). These naturally extend to multiple quantiles using composition. Basic composition says that, if we estimate each of $m$ quantiles via an $\tfrac{\eps}{m}$-differentially private algorithm, then we will obtain $\eps$-differential privacy overall for the set of $m$ quantiles. However, the cost of this generality is the smaller and more restrictive privacy budget $\tfrac{\eps}{m}$ (or roughly $\tfrac{\eps}{\sqrt{m}}$ for ``advanced'' composition). As a result, this approach yields significantly less accurate outcomes as $m$ grows. This is unfortunate, as many applications rely on multiple quantiles: returning to the opening paragraph, the income statistics use $m=4$ (quintiles), the birth weight statistics use $m=9$ (deciles), and the test score statistics use $m > 30$. Alternatively, there exist methods for computing a differentially private summary of the entire distribution from which any quantile can subsequently be estimated (see \cref{subsec:rel}). However, unless $m$ is very large, such summaries will usually contain more information than needed, reducing accuracy.

\subsection{Contributions}
\label{subsec:contrib}
\begin{enumerate}
    \item We give an instantiation of the exponential mechanism~\cite{MT07}, \algo, that produces an $\eps$-differentially private collection of $m$ quantile estimates in a single invocation (\cref{subsec:initial}).  This mechanism uses a utility function that has sensitivity $2$ no matter how many quantiles are requested, and does not need to divide $\eps$ based on the number of quantiles.
    \item We provide a dynamic program, related to algorithms used for inference in graphical models, to implement \algo~in time $O(mn\log(n) + m^2n)$ (\cref{subsec:finite}, \cref{subsec:fb})\arxiv{\footnote{The first version of this paper did not include the FFT optimization and thus had runtime $O(mn^2 + m^2n)$.}}. This significantly improves on naive sampling, which requires time $O(n^m)$. 
    \item We experimentally evaluate \algo~and find that it obtains much better accuracy than the existing state-of-the-art while remaining efficient enough to be practical for moderate dataset sizes (\cref{sec:experiments}).
\end{enumerate}

\subsection{Related Work}
\label{subsec:rel}
Discussion of single-quantile estimation algorithms appears in Section~\ref{sec:experiments}. At the other end of the spectrum, one can use private CDF estimation or private threshold release to estimate arbitrarily many quantiles. These approaches avoid splitting $\eps$ as $m$ grows but suffer from the need to set hyperparameters depending on the discretization of the domain and assumptions about the data distribution. Moreover, the best known algorithms for threshold release rely on several reductions that limit their practicality \cite{BNSV15, KLMNS20}. A common tree-based approach to CDF estimation is included in our experiments.

Our algorithm relies on dynamic programming to sample from the exponential mechanism.\narxiv{~\citet{BDB16}} \arxiv{Blocki, Datta, and Bonneau \cite{BDB16}} studied how to release the counts (but not identities) of items in a dataset by constructing a relaxation of the exponential mechanism and sampling from it using dynamic programming. However, their utility function more simply decomposes into individual terms without pairwise interactions, and it is not clear how this method can be applied to quantiles.

Finally, our dynamic program for sampling from \algo's exponential mechanism is related to inference algorithms for graphical models. Several papers have studied differential privacy with graphical models. However, this has typically meant studying private versions of graphical modeling tasks~\cite{WM10, BMSSHM17} or using graphical models as a step in private algorithms~\cite{MSM19}. Our paper departs from that past work in that its dynamic program, while related to the forward-backward algorithm, does not have any conceptual dependence on graphical models themselves.
\section{Preliminaries}
\label{sec:prelims}
We view databases $X, X'$ as multisets of elements from some data domain $\X$ where each individual contributes at most one element to the database. To reason about databases that are ``close'', differential privacy uses \emph{neighbors}.
\begin{definition}
    Databases $X$ and $X' \in \X^n$ are \emph{neighbors}, denoted $X \sim X'$, if they differ in at most one element.
\end{definition}
Note that we use the \emph{swap} definition of differential privacy; in contrast, the \emph{add-remove} definition allows the addition or removal (rather than exchange) of one element between neighboring databases. We do this for consistent evaluation against the smooth-sensitivity framework (see \cref{sec:app_comparisons}), which also uses swap differential privacy. However, we emphasize that our algorithm \algo~easily adapts to the add-remove framework (in fact, its sensitivity is lower under add-remove privacy).

With the notion of neighboring databases in hand, we can now define differential privacy.
\begin{definition}[\narxiv{\citet{DMNS06}}\arxiv{Dwork, McSherry, Nissim, and Smith~\cite{DMNS06}}]
    A randomized algorithm $\A \colon \X^* \to \Y$ is $(\eps,\delta)$-differentially private if, for every pair of neighboring databases $X, X'$ and every output subset $Y \subseteq \Y$, 
    $$\P{\A}{A(X) \in Y} \leq e^\eps\P{\A}{A(X') \in Y} + \delta.$$
    When $\delta > 0$, we say $\A$ satisfies \emph{approximate} differential privacy. If $\delta = 0$, we say $\A$ satisfies \emph{pure} differential privacy, and shorthand this as $\eps$-differential privacy (or $\eps$-DP).
\end{definition}
A key benefit of differential privacy is composition: an algorithm that relies on differentially private subroutines inherits an overall privacy guarantee by simply adding up the privacy guarantees of its components.
\begin{lemma}[\citet{DMNS06}]
\label{lem:comp}
    Let $\A_1, \ldots, \A_k$ be $k$ algorithms that respectively satisfy $(\eps_1, \delta_1)$-$, \ldots, (\eps_k, \delta_k)$-differential privacy. Then running $\A_1, \ldots, \A_k$ satisfies $\left(\sum_{i=1}^k \eps_i, \sum_{i=1}^k \delta_i\right)$-differential privacy.
\end{lemma}
We will use composition (or its ``advanced'' variants) when evaluating methods that estimate a set of $m$ quantiles by estimating each quantile individually. By \cref{lem:comp}, to achieve overall $\eps$-DP, it suffices to estimate each quantile under $\tfrac{\eps}{m}$-DP. However, since our algorithm \algo~estimates all quantiles in one invocation, it does not use composition.

We will also rely on the exponential mechanism, a common building block for differentially private algorithms.
\begin{definition}[\citet{MT07, DR14}]
\label{def:exp}
    Given utility function $u \colon \X^* \times O \to \mathbb{R}$ mapping $(\text{database}, \text{output})$ pairs to real-valued scores with $L_1$ sensitivity
    \narxiv{$\Delta_u = \max_{X \sim X', o \in O} |u(X, o) - u(X', o)|$,}
    \arxiv{\begin{equation*}\Delta_u = \max_{X \sim X', o \in O} |u(X, o) - u(X', o)|,\end{equation*}}
    the exponential mechanism $M$ has output distribution
    $$\P{M}{M(X) = o} \propto \exp\left(\frac{\eps u(X,o)}{2\Delta_u}\right),$$
    where $\propto$ elides the normalization factor.
\end{definition}
The exponential mechanism thus prioritizes a database's higher-utility outputs while remaining private.
\begin{lemma}[\citet{MT07}]
\label{lem:exp_dp}
    The mechanism described in \cref{def:exp} is $\eps$-DP.
\end{lemma}

The above material suffices to understand the bulk of our algorithm, \algo. The algorithms used for our experimental comparisons will also require some understanding of smooth sensitivity and concentrated differential privacy, but since these concepts will be relevant only as points of experimental comparison, we discuss them in \cref{sec:experiments}.
\section{\algo}
\label{sec:algo}
This section provides an exposition of our quantiles algorithm, \algo. Recall that our goal is to take as input quantiles $q_1 < q_2 < \ldots < q_m \in [0,1]$ and database $X$ and output quantile estimates $o_1, \ldots, o_m$ such that, for each $j \in [m]$, $\P{x \sim_U X}{x \leq o_j} \approx q_j$.

In \cref{subsec:initial}, we start with an instance of the exponential mechanism whose continuous output space makes sampling impractical. In \cref{subsec:finite}, we construct a mechanism with the same output distribution (and, importantly, the same privacy guarantees) and a bounded but inefficient sampling procedure. Finally, in \cref{subsec:fb} we modify our sampling procedure once more to produce an equivalent and polynomial time method, which we call \algo.

\subsection{Initial Solution}
\label{subsec:initial}
We start by formulating an instance of the exponential mechanism for our quantiles setting. First, we will require the algorithm user to input a lower bound $a$ and upper bound $b$ for the data domain.\footnote{Lower and upper bounds are also necessary for the private quantile algorithms that we compare to in our experiments. We find that choosing loose bounds $a$ and $b$ does not greatly affect utility (see experiments in \cref{sec:experiments}).}
We assume that all  $x \in X$ are in $[a,b]$; if this is not the case initially, then we clamp any outside points to $[a,b]$.  The output space is $O_\nearrow = \{(o_1, \ldots, o_m) \mid a \leq o_1 \leq \cdots \leq o_m \leq b\}$, the set of sequences of $m$ nondecreasing values from $[a,b]$. For a given $o = (o_1, \ldots, o_m)$, the utility function will compare the number of points in each proposed quantile interval $[o_{j-1}, o_j)$ to the expected number of points in the correct quantile interval.\footnote{Like the single-quantile exponential mechanism~\cite{S11}, this utility function works best when there are not large numbers of duplicate points (though it is private in all cases). This issue did not arise in our experiments here, but we have found that perturbing the data by a small amount of data-independent noise resolves it in practice.} We denote the number of data points between adjacent quantiles $q_{j-1}$ and $q_j$ by $n_j = (q_j - q_{j-1})n$. We fix $q_0 = 0$ and $q_{m+1} = 1$, so that $n_1 = q_1n$ and $n_{m+1} = (1 - q_m)n$. We also denote the number of data points from $X$ between any two values $u$ and $v$ by
$$n(u, v) = |\{x \in X \mid u \leq x < v\}|.$$
We can now define our utility function
$$u_Q(X, o) = -\sum_{j \in [m+1]} |n(o_{j-1}, o_j) - n_j|~,$$
where we fix $o_0 = a$ and $o_{m+1} = b+1$ (setting $o_{m+1}$ to a value strictly larger than $b$ simply ensures that points equal to $b$ are counted in the final term of the sum). $u_Q$ thus assigns highest utility to the true quantile values and lower utility to estimates that are far from the true quantile values.

\begin{restatable}{lemma}{AlgoSensitivity}
    \label{lem:algo_sensitivity}
    $u_Q$ has $L_1$ sensitivity $\Delta_{u_Q} = 2$.
\end{restatable}
\begin{proof}
    Fix an output $o$. Let $X$ and $X'$ be neighboring databases. Since we use swap differential privacy, $|X| = |X'|$, so $u_Q(X,o)$ and $u_Q(X',o)$ only differ in their $n(\cdot, \cdot)$ (respectively denoted $n'(\cdot, \cdot)$ for $X'$). Since $X$ and $X'$ are neighbors, there are at most two intervals $(o_{j-1}, o_j)$ and $(o_{j'-1}, o_{j'})$ on which $n$ and $n'$ differ, each by at most one. Thus $|u_Q(X,o) - u_Q(X',o)| \leq 2$.
\end{proof}
For add-remove privacy, the sensitivity is slightly lower at $\Delta_{u_Q} = 2[1 - \min_{j \in [m+1]} (q_j - q_{j-1})]$. A full proof of this and other results appears in \cref{subsec:proofs}.

The corresponding mechanism $M_Q$ has output density
\begin{equation}
\label{eq:Q_density}
f_Q(o) \propto \cdot \exp\left(\frac{\eps}{2\Delta_{u_Q}} \cdot u_Q(X, o)\right).
\end{equation}
Since this is an instantiation of the exponential mechanism, we can apply \cref{lem:exp_dp} to get:
\begin{lemma}
\label{lem:initial_dp}
    The mechanism $M_Q$ defined by output density $f_Q$ satisfies $\eps$-differential privacy.
\end{lemma}
However, as is typically a drawback of the exponential mechanism, it is not clear how to efficiently sample from this distribution, which is defined over a continuous $m$-dimensional space. The following sections address this issue. Since the output distribution itself remains fixed through these sampling procedure changes, these improvements will preserve the privacy guarantee of \cref{lem:initial_dp}. The remaining proofs will therefore focus on verifying that subsequent sampling procedures still sample according to \cref{eq:Q_density}.

\subsection{Finite Sampling Improvement}
\label{subsec:finite}
In this section, we describe how to sample from the continuous distribution defined by $M_Q$ by first sampling from an intermediate discrete distribution.  This is similar to the single-quantile sampling technique given by~\citet{S11} (see their Algorithm 2).  The basic idea is that we split the sampling process into three steps:
\begin{enumerate}
    \item Sample $m$ intervals from the set of intervals between data points. 
    \vspace{-5pt}
    \item Take a uniform random sample from each of the $m$ sampled intervals.
    \vspace{-5pt}
    \item Output the samples in increasing order.
\end{enumerate}

This will require some additional notation. Denote the elements of $X$ in nondecreasing order by $x_1 \leq \cdots \leq x_n$, fix $x_0 = a$ and $x_{n+1} = b$, and let $I = \{0, \dots, n\}$, where we associate $i \in I$ with the interval between points $x_i$ and $x_{i+1}$. Define $S_\nearrow$ to be the set of nondecreasing sequences of $m$ intervals,
$$S_\nearrow = \{(i_1, \ldots, i_m) \mid i_1, \ldots, i_m \in I, i_1 \leq \cdots \leq i_m\}.$$
$S_\nearrow$ will be the discrete output space for the first sampling step above. We can define a utility function $u_{Q'}$ on $s = (i_1, \ldots, i_m) \in S_\nearrow$ by slightly modifying $u_Q$:
$$u_{Q'}(X, s) = -\sum_{j \in [m+1]} |(i_j - i_{j-1}) - n_j|,$$
where we fix $i_0 = 0$ and $i_{m+1} = n$.

In order to reproduce $M_Q$ from \cref{subsec:initial}, our sequence sampler will also need to weight each sequence $s$ by the total measure of the outputs $o \in O_\nearrow$ that can be sampled from $s$ in the second step. This is nontrivial due to the ordering constraint on $o$: if an interval appears twice in $s$, the measure of corresponding outputs must be halved to account for the fact that the two corresponding samples in the second step can appear in either order, but will be mapped to a fixed increasing order in the third step. In general, if an interval appears $k$ times, the measure must be scaled by a factor of $1/k!$, the volume of the standard $k$-simplex. We account for this by dividing by the scale function
\[
    \gamma(s) = \prod_{i \in I} \num{i}{s}!~,
\]
where $\num{i}{s}$ is the number of times $i$ appears in $s$ and we take $0! = 1$.
    
The mechanism $M_{Q'}$ is now defined as follows:
\begin{enumerate}
    \item Draw $s = (i_1, \ldots, i_m)$ according to
    $$\P{M_{Q'}}{s} \propto \exp\left(\frac{\eps \cdot u_{Q'}(X, s)}{2\Delta_{u_Q}}\right) \cdot 
                \frac{\prod_{j=1}^m (x_{i_j + 1} - x_{i_j})}{\gamma(s)}.$$
    \item For $j \in [m]$, draw $o_j$ uniformly at random from $[x_{i_j}, x_{i_j+1})$.
    \item Output $o_1, \ldots, o_m$ in increasing order.
\end{enumerate}

It remains to verify that $M_{Q'}$ actually matches $M_Q$.

\begin{restatable}{lemma}{MQPrime}
\label{lem:m_q_prime}
    $M_{Q'}$ has the same output distribution as $M_Q$.
\end{restatable}
\begin{proof}[Proof Sketch (see \cref{subsec:proofs} for full proof)]
    Given potential outputs $o$ and $o'$, if the corresponding quantile estimates fall into the same intervals between data points in $X$, then the counts $n(\cdot, \cdot)$ are unchanged and $u_Q(X,o) = u_Q(X,o')$. Since $u_Q$ is constant over intervals between data points, it is equivalent to sample those intervals and then sample points uniformly at random from the chosen intervals. The only complication is accounting for the $\gamma$ scaling introduced by repeated intervals.
\end{proof}
The benefit of $M_{Q'}$ over $M_Q$ is that the first step samples from a finite space, and the second sampling step is simply uniform sampling. However, the size of the space for the first step is still $O(n^m)$, which remains impractical for all but the smallest datasets. In the next section, we develop a dynamic programming algorithm that allows us to sample from $\Po{M_{Q'}}$ in time $O(mn\log(n) + m^2n)$.

\subsection{Polynomial Sampling Improvement}
\label{subsec:fb}

\narxiv{\begin{figure*}
\begin{center}
  % https://www.mathcha.io/editor/5Q9dqtYniVvh94j7XGurZWZWCDmqoYZFyyMNEz

% Gradient Info
  
\tikzset {_l5j7f9roy/.code = {\pgfsetadditionalshadetransform{ \pgftransformshift{\pgfpoint{0 bp } { 0 bp }  }  \pgftransformscale{1 }  }}}
\pgfdeclareradialshading{_d0fts3u60}{\pgfpoint{0bp}{0bp}}{rgb(0bp)=(0.81,0.91,0.98);
rgb(0bp)=(0.81,0.91,0.98);
rgb(25bp)=(0.39,0.58,0.76);
rgb(400bp)=(0.39,0.58,0.76)}

% Gradient Info
  
\tikzset {_sshzaa6k6/.code = {\pgfsetadditionalshadetransform{ \pgftransformshift{\pgfpoint{0 bp } { 0 bp }  }  \pgftransformscale{1 }  }}}
\pgfdeclareradialshading{_0zx7ts8p6}{\pgfpoint{0bp}{0bp}}{rgb(0bp)=(0.81,0.91,0.98);
rgb(0bp)=(0.81,0.91,0.98);
rgb(25bp)=(0.39,0.58,0.76);
rgb(400bp)=(0.39,0.58,0.76)}

% Gradient Info
  
\tikzset {_h6i2mcdvz/.code = {\pgfsetadditionalshadetransform{ \pgftransformshift{\pgfpoint{0 bp } { 0 bp }  }  \pgftransformscale{1 }  }}}
\pgfdeclareradialshading{_slo18ixov}{\pgfpoint{0bp}{0bp}}{rgb(0bp)=(0.81,0.91,0.98);
rgb(0bp)=(0.81,0.91,0.98);
rgb(25bp)=(0.39,0.58,0.76);
rgb(400bp)=(0.39,0.58,0.76)}

% Gradient Info
  
\tikzset {_vg8wxb0yn/.code = {\pgfsetadditionalshadetransform{ \pgftransformshift{\pgfpoint{0 bp } { 0 bp }  }  \pgftransformscale{1 }  }}}
\pgfdeclareradialshading{_uy8uhc3hu}{\pgfpoint{0bp}{0bp}}{rgb(0bp)=(0.81,0.91,0.98);
rgb(0bp)=(0.81,0.91,0.98);
rgb(25bp)=(0.39,0.58,0.76);
rgb(400bp)=(0.39,0.58,0.76)}

% Gradient Info
  
\tikzset {_utuhv405x/.code = {\pgfsetadditionalshadetransform{ \pgftransformshift{\pgfpoint{0 bp } { 0 bp }  }  \pgftransformscale{1 }  }}}
\pgfdeclareradialshading{_qcoc1xl19}{\pgfpoint{0bp}{0bp}}{rgb(0bp)=(0.95,0.91,0.4);
rgb(0bp)=(0.95,0.91,0.4);
rgb(25bp)=(1,0.71,0.27);
rgb(400bp)=(1,0.71,0.27)}

% Gradient Info
  
\tikzset {_wg4bg49sc/.code = {\pgfsetadditionalshadetransform{ \pgftransformshift{\pgfpoint{0 bp } { 0 bp }  }  \pgftransformscale{1 }  }}}
\pgfdeclareradialshading{_x3qcain8v}{\pgfpoint{0bp}{0bp}}{rgb(0bp)=(0.95,0.91,0.4);
rgb(0bp)=(0.95,0.91,0.4);
rgb(25bp)=(1,0.71,0.27);
rgb(400bp)=(1,0.71,0.27)}
\tikzset{every picture/.style={line width=0.75pt}} %set default line width to 0.75pt        

\begin{tikzpicture}[x=0.75pt,y=0.75pt,yscale=-1,xscale=1]
%uncomment if require: \path (0,290); %set diagram left start at 0, and has height of 290

%Shape: Ellipse [id:dp5013958740121709] 
\draw  [draw opacity=0][shading=_d0fts3u60,_l5j7f9roy] (165,124.86) .. controls (165,115.08) and (173.18,107.15) .. (183.28,107.15) .. controls (193.37,107.15) and (201.55,115.08) .. (201.55,124.86) .. controls (201.55,134.64) and (193.37,142.57) .. (183.28,142.57) .. controls (173.18,142.57) and (165,134.64) .. (165,124.86) -- cycle ;
%Shape: Ellipse [id:dp4196540723485509] 
\draw  [draw opacity=0][shading=_0zx7ts8p6,_sshzaa6k6] (265.88,124.86) .. controls (265.88,115.08) and (274.07,107.15) .. (284.16,107.15) .. controls (294.25,107.15) and (302.43,115.08) .. (302.43,124.86) .. controls (302.43,134.64) and (294.25,142.57) .. (284.16,142.57) .. controls (274.07,142.57) and (265.88,134.64) .. (265.88,124.86) -- cycle ;
%Straight Lines [id:da9214235548117129] 
\draw [line width=3]    (200.95,124.86) -- (265.88,124.86) ;
%Shape: Ellipse [id:dp6123347414206908] 
\draw  [draw opacity=0][shading=_slo18ixov,_h6i2mcdvz] (366.77,124.86) .. controls (366.77,115.08) and (374.95,107.15) .. (385.04,107.15) .. controls (395.13,107.15) and (403.32,115.08) .. (403.32,124.86) .. controls (403.32,134.64) and (395.13,142.57) .. (385.04,142.57) .. controls (374.95,142.57) and (366.77,134.64) .. (366.77,124.86) -- cycle ;
%Shape: Ellipse [id:dp3591424435497559] 
\draw  [draw opacity=0][shading=_uy8uhc3hu,_vg8wxb0yn] (457.65,124.86) .. controls (457.65,115.08) and (465.83,107.15) .. (475.92,107.15) .. controls (486.02,107.15) and (494.2,115.08) .. (494.2,124.86) .. controls (494.2,134.64) and (486.02,142.57) .. (475.92,142.57) .. controls (465.83,142.57) and (457.65,134.64) .. (457.65,124.86) -- cycle ;
%Straight Lines [id:da7004957741062938] 
\draw [line width=3]    (302.43,124.86) -- (366.77,124.86) ;
%Straight Lines [id:da29106218722121713] 
\draw [line width=3]    (494.2,124.86) -- (603,125.29) ;
%Straight Lines [id:da6056365667661876] 
\draw [line width=3]    (98.6,124.9) -- (165,124.76) ;
%Shape: Circle [id:dp8411692132456796] 
\draw  [fill={rgb, 255:red, 0; green, 0; blue, 0 }  ,fill opacity=1 ] (418.84,126.4) .. controls (418.84,125.57) and (419.52,124.9) .. (420.34,124.9) .. controls (421.17,124.9) and (421.84,125.57) .. (421.84,126.4) .. controls (421.84,127.23) and (421.17,127.9) .. (420.34,127.9) .. controls (419.52,127.9) and (418.84,127.23) .. (418.84,126.4) -- cycle ;
%Shape: Circle [id:dp7392939681284698] 
\draw  [fill={rgb, 255:red, 0; green, 0; blue, 0 }  ,fill opacity=1 ] (428.76,126.4) .. controls (428.76,125.57) and (429.43,124.9) .. (430.26,124.9) .. controls (431.08,124.9) and (431.76,125.57) .. (431.76,126.4) .. controls (431.76,127.23) and (431.08,127.9) .. (430.26,127.9) .. controls (429.43,127.9) and (428.76,127.23) .. (428.76,126.4) -- cycle ;
%Shape: Circle [id:dp8255904276173709] 
\draw  [fill={rgb, 255:red, 0; green, 0; blue, 0 }  ,fill opacity=1 ] (438.8,126.4) .. controls (438.8,125.57) and (439.47,124.9) .. (440.3,124.9) .. controls (441.13,124.9) and (441.8,125.57) .. (441.8,126.4) .. controls (441.8,127.23) and (441.13,127.9) .. (440.3,127.9) .. controls (439.47,127.9) and (438.8,127.23) .. (438.8,126.4) -- cycle ;
%Rounded Rect [id:dp9097056158031704] 
\draw  [draw opacity=0][shading=_qcoc1xl19,_utuhv405x] (46.5,115) .. controls (46.5,111.37) and (49.44,108.43) .. (53.07,108.43) -- (92.43,108.43) .. controls (96.06,108.43) and (99,111.37) .. (99,115) -- (99,134.71) .. controls (99,138.34) and (96.06,141.29) .. (92.43,141.29) -- (53.07,141.29) .. controls (49.44,141.29) and (46.5,138.34) .. (46.5,134.71) -- cycle ;
%Rounded Rect [id:dp3260242070263233] 
\draw  [draw opacity=0][shading=_x3qcain8v,_wg4bg49sc] (603,114.98) .. controls (603,111.35) and (605.95,108.4) .. (609.58,108.4) -- (659.75,108.4) .. controls (663.39,108.4) and (666.33,111.35) .. (666.33,114.98) -- (666.33,134.72) .. controls (666.33,138.35) and (663.39,141.3) .. (659.75,141.3) -- (609.58,141.3) .. controls (605.95,141.3) and (603,138.35) .. (603,134.72) -- cycle ;

% Text Node
\draw (178.01,119.14) node [anchor=north west][inner sep=0.75pt]  [font=\footnotesize]  {$i_{1}$};
% Text Node
\draw (277.89,119.14) node [anchor=north west][inner sep=0.75pt]  [font=\footnotesize]  {$i_{2}$};
% Text Node
\draw (201.71,107.35) node [anchor=north west][inner sep=0.75pt]  [font=\footnotesize]  {$\phi ( i_{1} ,\ i_{2} ,\ 2)$};
% Text Node
\draw (378.77,119.14) node [anchor=north west][inner sep=0.75pt]  [font=\footnotesize]  {$i_{3}$};
% Text Node
\draw (469.66,119.14) node [anchor=north west][inner sep=0.75pt]  [font=\footnotesize]  {$i_{m}$};
% Text Node
\draw (301.2,107.35) node [anchor=north west][inner sep=0.75pt]  [font=\footnotesize]  {$\phi ( i_{2} ,\ i_{3} ,\ 3)$};
% Text Node
\draw (55.62,119.1) node [anchor=north west][inner sep=0.75pt]  [font=\footnotesize]  {$i_{0} =0$};
% Text Node
\draw (494.08,105.95) node [anchor=north west][inner sep=0.75pt]  [font=\footnotesize]  {$\phi ( i_{m} ,\ i_{m+1} ,\ m+1)$};
% Text Node
\draw (101.11,107.35) node [anchor=north west][inner sep=0.75pt]  [font=\footnotesize]  {$\phi ( i_{0} ,\ i_{1} ,\ 1)$};
% Text Node
\draw (609.95,119.1) node [anchor=north west][inner sep=0.75pt]  [font=\footnotesize]  {$i_{m+1} =n$};
% Text Node
\draw (58.51,90.95) node [anchor=north west][inner sep=0.75pt]  [font=\footnotesize]  {$\tau ( i_{0})$};
% Text Node
\draw (168.79,90.9) node [anchor=north west][inner sep=0.75pt]  [font=\footnotesize]  {$\tau ( i_{1})$};
% Text Node
\draw (271.37,90.9) node [anchor=north west][inner sep=0.75pt]  [font=\footnotesize]  {$\tau ( i_{2})$};
% Text Node
\draw (370.65,90.9) node [anchor=north west][inner sep=0.75pt]  [font=\footnotesize]  {$\tau ( i_{3})$};
% Text Node
\draw (459.79,90.9) node [anchor=north west][inner sep=0.75pt]  [font=\footnotesize]  {$\tau ( i_{m})$};

\end{tikzpicture}
  \caption{Illustration of pairwise dependencies for interval sequence $s = (i_1, \ldots, i_m)$.}
  \label{fig:markov_chain}
\end{center}
\end{figure*}
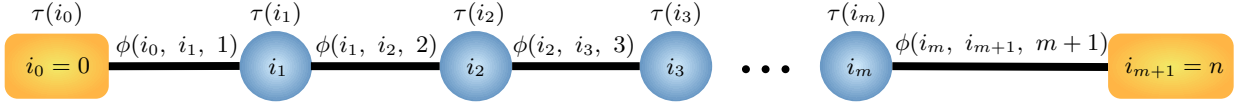}

Notice that the bulk of our probability distribution over sequences $s = (i_1, \ldots, i_m)$ can be decomposed as a product of scores, where each score depends only on adjacent intervals $i_{j-1}$ and $i_j$. In particular,
$$\P{M_{Q'}}{s} \propto \frac{1}{\gamma(s)} \prod_{j \in [m+1]} \phi(i_{j-1}, i_j, j)\prod_{j \in [m]} \tau(i_j)~,$$
where for $i \leq i'$ and $j \in [m+1]$ we define
\begin{align*}
    \phi(i, i', j) =&\ \exp\left(-\frac{\eps}{2\Delta_{u_Q}} |(i' - i) - n_j|\right) \\
    \tau(i) =&\ x_{i+1} - x_{i}.
\end{align*}
For $i > i'$ and any $j$, $\phi(i, i', j) = 0$. \cref{fig:markov_chain} illustrates this structure graphically, suggesting a dynamic programming algorithm similar to the ``forward-backward'' algorithm  from the graphical models literature (see, e.g., Chapter 15 of \citet{RN10}). 

\arxiv{\begin{figure*}[!htbp]
\begin{center}
  % https://www.mathcha.io/editor/5Q9dqtYniVvh94j7XGurZWZWCDmqoYZFyyMNEz

% Gradient Info
  
\tikzset {_l5j7f9roy/.code = {\pgfsetadditionalshadetransform{ \pgftransformshift{\pgfpoint{0 bp } { 0 bp }  }  \pgftransformscale{1 }  }}}
\pgfdeclareradialshading{_d0fts3u60}{\pgfpoint{0bp}{0bp}}{rgb(0bp)=(0.81,0.91,0.98);
rgb(0bp)=(0.81,0.91,0.98);
rgb(25bp)=(0.39,0.58,0.76);
rgb(400bp)=(0.39,0.58,0.76)}

% Gradient Info
  
\tikzset {_sshzaa6k6/.code = {\pgfsetadditionalshadetransform{ \pgftransformshift{\pgfpoint{0 bp } { 0 bp }  }  \pgftransformscale{1 }  }}}
\pgfdeclareradialshading{_0zx7ts8p6}{\pgfpoint{0bp}{0bp}}{rgb(0bp)=(0.81,0.91,0.98);
rgb(0bp)=(0.81,0.91,0.98);
rgb(25bp)=(0.39,0.58,0.76);
rgb(400bp)=(0.39,0.58,0.76)}

% Gradient Info
  
\tikzset {_h6i2mcdvz/.code = {\pgfsetadditionalshadetransform{ \pgftransformshift{\pgfpoint{0 bp } { 0 bp }  }  \pgftransformscale{1 }  }}}
\pgfdeclareradialshading{_slo18ixov}{\pgfpoint{0bp}{0bp}}{rgb(0bp)=(0.81,0.91,0.98);
rgb(0bp)=(0.81,0.91,0.98);
rgb(25bp)=(0.39,0.58,0.76);
rgb(400bp)=(0.39,0.58,0.76)}

% Gradient Info
  
\tikzset {_vg8wxb0yn/.code = {\pgfsetadditionalshadetransform{ \pgftransformshift{\pgfpoint{0 bp } { 0 bp }  }  \pgftransformscale{1 }  }}}
\pgfdeclareradialshading{_uy8uhc3hu}{\pgfpoint{0bp}{0bp}}{rgb(0bp)=(0.81,0.91,0.98);
rgb(0bp)=(0.81,0.91,0.98);
rgb(25bp)=(0.39,0.58,0.76);
rgb(400bp)=(0.39,0.58,0.76)}

% Gradient Info
  
\tikzset {_utuhv405x/.code = {\pgfsetadditionalshadetransform{ \pgftransformshift{\pgfpoint{0 bp } { 0 bp }  }  \pgftransformscale{1 }  }}}
\pgfdeclareradialshading{_qcoc1xl19}{\pgfpoint{0bp}{0bp}}{rgb(0bp)=(0.95,0.91,0.4);
rgb(0bp)=(0.95,0.91,0.4);
rgb(25bp)=(1,0.71,0.27);
rgb(400bp)=(1,0.71,0.27)}

% Gradient Info
  
\tikzset {_wg4bg49sc/.code = {\pgfsetadditionalshadetransform{ \pgftransformshift{\pgfpoint{0 bp } { 0 bp }  }  \pgftransformscale{1 }  }}}
\pgfdeclareradialshading{_x3qcain8v}{\pgfpoint{0bp}{0bp}}{rgb(0bp)=(0.95,0.91,0.4);
rgb(0bp)=(0.95,0.91,0.4);
rgb(25bp)=(1,0.71,0.27);
rgb(400bp)=(1,0.71,0.27)}
\tikzset{every picture/.style={line width=0.75pt}} %set default line width to 0.75pt        

\begin{tikzpicture}[x=0.75pt,y=0.75pt,yscale=-1,xscale=1]
%uncomment if require: \path (0,290); %set diagram left start at 0, and has height of 290

%Shape: Ellipse [id:dp5013958740121709] 
\draw  [draw opacity=0][shading=_d0fts3u60,_l5j7f9roy] (165,124.86) .. controls (165,115.08) and (173.18,107.15) .. (183.28,107.15) .. controls (193.37,107.15) and (201.55,115.08) .. (201.55,124.86) .. controls (201.55,134.64) and (193.37,142.57) .. (183.28,142.57) .. controls (173.18,142.57) and (165,134.64) .. (165,124.86) -- cycle ;
%Shape: Ellipse [id:dp4196540723485509] 
\draw  [draw opacity=0][shading=_0zx7ts8p6,_sshzaa6k6] (265.88,124.86) .. controls (265.88,115.08) and (274.07,107.15) .. (284.16,107.15) .. controls (294.25,107.15) and (302.43,115.08) .. (302.43,124.86) .. controls (302.43,134.64) and (294.25,142.57) .. (284.16,142.57) .. controls (274.07,142.57) and (265.88,134.64) .. (265.88,124.86) -- cycle ;
%Straight Lines [id:da9214235548117129] 
\draw [line width=3]    (200.95,124.86) -- (265.88,124.86) ;
%Shape: Ellipse [id:dp6123347414206908] 
\draw  [draw opacity=0][shading=_slo18ixov,_h6i2mcdvz] (366.77,124.86) .. controls (366.77,115.08) and (374.95,107.15) .. (385.04,107.15) .. controls (395.13,107.15) and (403.32,115.08) .. (403.32,124.86) .. controls (403.32,134.64) and (395.13,142.57) .. (385.04,142.57) .. controls (374.95,142.57) and (366.77,134.64) .. (366.77,124.86) -- cycle ;
%Shape: Ellipse [id:dp3591424435497559] 
\draw  [draw opacity=0][shading=_uy8uhc3hu,_vg8wxb0yn] (457.65,124.86) .. controls (457.65,115.08) and (465.83,107.15) .. (475.92,107.15) .. controls (486.02,107.15) and (494.2,115.08) .. (494.2,124.86) .. controls (494.2,134.64) and (486.02,142.57) .. (475.92,142.57) .. controls (465.83,142.57) and (457.65,134.64) .. (457.65,124.86) -- cycle ;
%Straight Lines [id:da7004957741062938] 
\draw [line width=3]    (302.43,124.86) -- (366.77,124.86) ;
%Straight Lines [id:da29106218722121713] 
\draw [line width=3]    (494.2,124.86) -- (603,125.29) ;
%Straight Lines [id:da6056365667661876] 
\draw [line width=3]    (98.6,124.9) -- (165,124.76) ;
%Shape: Circle [id:dp8411692132456796] 
\draw  [fill={rgb, 255:red, 0; green, 0; blue, 0 }  ,fill opacity=1 ] (418.84,126.4) .. controls (418.84,125.57) and (419.52,124.9) .. (420.34,124.9) .. controls (421.17,124.9) and (421.84,125.57) .. (421.84,126.4) .. controls (421.84,127.23) and (421.17,127.9) .. (420.34,127.9) .. controls (419.52,127.9) and (418.84,127.23) .. (418.84,126.4) -- cycle ;
%Shape: Circle [id:dp7392939681284698] 
\draw  [fill={rgb, 255:red, 0; green, 0; blue, 0 }  ,fill opacity=1 ] (428.76,126.4) .. controls (428.76,125.57) and (429.43,124.9) .. (430.26,124.9) .. controls (431.08,124.9) and (431.76,125.57) .. (431.76,126.4) .. controls (431.76,127.23) and (431.08,127.9) .. (430.26,127.9) .. controls (429.43,127.9) and (428.76,127.23) .. (428.76,126.4) -- cycle ;
%Shape: Circle [id:dp8255904276173709] 
\draw  [fill={rgb, 255:red, 0; green, 0; blue, 0 }  ,fill opacity=1 ] (438.8,126.4) .. controls (438.8,125.57) and (439.47,124.9) .. (440.3,124.9) .. controls (441.13,124.9) and (441.8,125.57) .. (441.8,126.4) .. controls (441.8,127.23) and (441.13,127.9) .. (440.3,127.9) .. controls (439.47,127.9) and (438.8,127.23) .. (438.8,126.4) -- cycle ;
%Rounded Rect [id:dp9097056158031704] 
\draw  [draw opacity=0][shading=_qcoc1xl19,_utuhv405x] (46.5,115) .. controls (46.5,111.37) and (49.44,108.43) .. (53.07,108.43) -- (92.43,108.43) .. controls (96.06,108.43) and (99,111.37) .. (99,115) -- (99,134.71) .. controls (99,138.34) and (96.06,141.29) .. (92.43,141.29) -- (53.07,141.29) .. controls (49.44,141.29) and (46.5,138.34) .. (46.5,134.71) -- cycle ;
%Rounded Rect [id:dp3260242070263233] 
\draw  [draw opacity=0][shading=_x3qcain8v,_wg4bg49sc] (603,114.98) .. controls (603,111.35) and (605.95,108.4) .. (609.58,108.4) -- (659.75,108.4) .. controls (663.39,108.4) and (666.33,111.35) .. (666.33,114.98) -- (666.33,134.72) .. controls (666.33,138.35) and (663.39,141.3) .. (659.75,141.3) -- (609.58,141.3) .. controls (605.95,141.3) and (603,138.35) .. (603,134.72) -- cycle ;

% Text Node
\draw (178.01,119.14) node [anchor=north west][inner sep=0.75pt]  [font=\footnotesize]  {$i_{1}$};
% Text Node
\draw (277.89,119.14) node [anchor=north west][inner sep=0.75pt]  [font=\footnotesize]  {$i_{2}$};
% Text Node
\draw (201.71,107.35) node [anchor=north west][inner sep=0.75pt]  [font=\footnotesize]  {$\phi ( i_{1} ,\ i_{2} ,\ 2)$};
% Text Node
\draw (378.77,119.14) node [anchor=north west][inner sep=0.75pt]  [font=\footnotesize]  {$i_{3}$};
% Text Node
\draw (469.66,119.14) node [anchor=north west][inner sep=0.75pt]  [font=\footnotesize]  {$i_{m}$};
% Text Node
\draw (301.2,107.35) node [anchor=north west][inner sep=0.75pt]  [font=\footnotesize]  {$\phi ( i_{2} ,\ i_{3} ,\ 3)$};
% Text Node
\draw (55.62,119.1) node [anchor=north west][inner sep=0.75pt]  [font=\footnotesize]  {$i_{0} =0$};
% Text Node
\draw (494.08,105.95) node [anchor=north west][inner sep=0.75pt]  [font=\footnotesize]  {$\phi ( i_{m} ,\ i_{m+1} ,\ m+1)$};
% Text Node
\draw (101.11,107.35) node [anchor=north west][inner sep=0.75pt]  [font=\footnotesize]  {$\phi ( i_{0} ,\ i_{1} ,\ 1)$};
% Text Node
\draw (609.95,119.1) node [anchor=north west][inner sep=0.75pt]  [font=\footnotesize]  {$i_{m+1} =n$};
% Text Node
\draw (58.51,90.95) node [anchor=north west][inner sep=0.75pt]  [font=\footnotesize]  {$\tau ( i_{0})$};
% Text Node
\draw (168.79,90.9) node [anchor=north west][inner sep=0.75pt]  [font=\footnotesize]  {$\tau ( i_{1})$};
% Text Node
\draw (271.37,90.9) node [anchor=north west][inner sep=0.75pt]  [font=\footnotesize]  {$\tau ( i_{2})$};
% Text Node
\draw (370.65,90.9) node [anchor=north west][inner sep=0.75pt]  [font=\footnotesize]  {$\tau ( i_{3})$};
% Text Node
\draw (459.79,90.9) node [anchor=north west][inner sep=0.75pt]  [font=\footnotesize]  {$\tau ( i_{m})$};

\end{tikzpicture}
  \caption{Illustration of pairwise dependencies for interval sequence $s = (i_1, \ldots, i_m)$.}
  \label{fig:markov_chain}
\end{center}
\end{figure*}}

Unfortunately, $\gamma(s)$ does not factor in the same way. However, it has its own special structure: since $s$ is required to be nondecreasing, $\gamma(s)$ decomposes over contiguous constant \emph{subsequences} of $s$. We will use this to design an efficient dynamic programming algorithm for sampling $\Po{M_{Q'}}$.

Define the function $\alpha \colon [m] \times I \times [m] \to \mathbb{R}$ so that $\alpha(j,i,k)$ is the total unnormalized probability mass for prefix sequences of length $j$ that end with exactly $k$ copies of the interval $i$. For all $i \in I$, let $\alpha(1, i, 1) = \phi(0, i, 1) \tau(i)$ and $\alpha(1, i, k) = 0$ for $k > 1$. Now, for $j = 2, \ldots, m$, we have the following recursion for all $i \in I$:
\begin{align*}
    \alpha(j, i, 1) =&\ \tau(i) \sum_{i' < i} \phi(i', i, j) \sum_{k < j} \alpha(j-1, i', k) \\
    \alpha(j, i, k > 1) =&\ \tau(i) \cdot \phi(i, i, j) \cdot \alpha(j-1, i, k-1) / k
\end{align*}
Intuitively, if the sequence ends with a single $i$, we need to sum over all possible preceding intervals $i'$, which could have been repeated up to $k < j$ times. On the other hand, if the sequence ends with more than one $i$, we know that the preceding interval was also $i$, and we simply divide by $k$ to account for the scale function $\gamma$.

Having computed $\alpha(\cdot, \cdot, \cdot)$, we can now use these quantities to sample in the reverse direction as follows. First, draw a pair
\narxiv{$(i, k) \propto \alpha(m, i, k) \phi(i, n, m+1)$}
\arxiv{\[(i, k) \propto \alpha(m, i, k) \phi(i, n, m+1)\]}
(the $\phi$ term accounts for the final edge in the graph; see \cref{subsec:proofs} for details). This determines that the last $k$ sampled intervals are equal to $i$. We can then draw another pair
\narxiv{$(i' < i, k') \propto \alpha(m-k, i', k') \phi(i', i, m-k+1)$,}
\arxiv{\[(i' < i, k') \propto \alpha(m-k, i', k') \phi(i', i, m-k+1),\]}
which determines that the last $k'$ remaining intervals in the sequence are $i'$, and so on until we have a complete sample.

We will verify that this procedure actually samples from the correct distribution in the proof of Theorem~\ref{thm:algo}. For now, we turn to an optimized version of this procedure, presented in \cref{alg:algo}. The main optimization leverages the structure of $\phi(\cdot, \cdot, j)$: fixing $j$, $\phi(i, i', j)$ depends only on $i' - i$. $\phi(\cdot, \cdot, j)$ is therefore a matrix with constant diagonals, i.e. a Toeplitz matrix. A key benefit of $n \times n$ Toeplitz matrices is that matrix-vector multiplication can be implemented in time $O(n\log(n))$ using the Fast Fourier Transform (see, e.g.,~\cite{B19}). This becomes useful to us once we rewrite the computation of $\alpha(j, \cdot, \cdot)$ using
\begin{align*}
    \hat\alpha(j-1, \cdot) =&\ \sum_{k < j} \alpha(j-1, \cdot, k) \\
    \alpha(j, \cdot, 1) =&\ \tau(\cdot) \times \left(\phi(\cdot, \cdot, j)^T\hat \alpha(j-1, \cdot)^T\right)
\end{align*} 
where $\times$ denotes element-wise product. This reduces each computation of $\alpha(j, \cdot, 1)$ in Line 9 of \cref{alg:algo} to time $O(n\log(n))$ and space $O(n)$.

In total, we spend time $O(m^2n)$ computing $\hat \alpha(\cdot, \cdot)$, time $O(mn\log(n))$ computing $\alpha(\cdot, \cdot, 1)$, and time $O(m^2)$ computing $\alpha(\cdot, \cdot, k)$ for $k > 1$. The result is overall time $O(mn\log(n) + m^2n)$. The space analysis essentially reduces to the space needed to store $\alpha$ while computing $\phi$ as needed. Details appear in the proof of Theorem~\ref{thm:algo}.

\begin{restatable}{theorem}{Algo}
\label{thm:algo}
    \algo~satisfies $\eps$-differential privacy, takes time $O(mn\log(n) + m^2n)$, and uses space $O(m^2n)$.
\end{restatable}

\begin{algorithm}
\begin{algorithmic}[1]
   \STATE {\bfseries Input:} sorted $X = (x_1 \leq \ldots \leq x_n)$ clamped to data range $[a,b]$, quantiles $q_1, \ldots, q_m$, privacy parameter $\eps$
   \STATE Set $x_0 = a$, $x_{n+1} = b$, and $\Delta_{u_Q} = 2$
   \STATE Set $I = \{0,\dots,n\}$, $i_0 = 0$, and $i_{m+1} = n$
   \FOR{$i \in I$}
        \STATE Set $\alpha(1, i, 1) = \phi(0, i, 1) \tau(i)$
        \STATE $\qquad\qquad\quad\;\,= \exp\left(-\frac{\eps}{2\Delta_{u_Q}} |i - n_1|\right) \cdot (x_{i+1} - x_i)$
   \ENDFOR
   \FOR{$j = 2, \ldots, m$}
        \FOR{$i \in I$}
            \STATE Set $\hat\alpha(j-1, i) = \sum_{k < j} \alpha(j-1,i,k)$
        \ENDFOR
        \STATE Set $\alpha(j, \cdot, 1) = \tau(\cdot) \times \left(\phi(\cdot, \cdot, j)^T\hat \alpha(j-1, \cdot)^T\right)$
        \FOR{$k=2, \ldots, j$}
            \FOR{$i \in I$}
                \STATE Set $\alpha(j, i, k) = \tau(i)  \phi(i, i, j)  \alpha(j-1, i, k-1)/k$
            \ENDFOR
        \ENDFOR
   \ENDFOR
   \STATE Sample $(i,k) \propto \alpha(m,i,k) \phi(i, n, m+1)$
   \STATE Set $i_{m-k+1}, \dots, i_m = i$, and $j = m-k$
   \WHILE{$j > 0$}
        \STATE Sample $(i,k) \propto \alpha(j,i,k) \phi(i, i_{j+1}, j+1)$
        \STATE Set $i_{j-k+1}, \dots, i_j = i$, and $j = j-k$
   \ENDWHILE  
   \STATE Output uniform samples $\{o_j \sim_U [x_{i_j}, x_{i_j + 1})\}_{j = 1}^m$ in increasing order
\end{algorithmic}
\caption{Pseudocode for \algo}
\label{alg:algo}
\end{algorithm}

\textbf{Numerical improvements.} Note that the quantities involved in computing $\phi$ and $\alpha$ may be quite small, so we implement \algo~using logarithmic quantities to avoid underflow errors in our experiments. This is a common trick and is a numerical rather than algorithmic change, but for completeness we include its details in \cref{subsec:log_trick}. After computing these quantities, to avoid underflow in our final sampling steps, we use a ``racing'' sampling method that was previously developed for single-quantile exponential mechanisms. Since this is again a numerical improvement, details appear in \cref{subsec:racing}.

\textbf{Connection to graphical models.} As mentioned above, the dynamic program in \cref{alg:algo} is similar to the forward-backward algorithm from the graphical models literature, modulo accounting for $\gamma(s)$. In graphical models, it is often necessary to compute the probability of a sequence of hidden states. This requires normalizing by a sum of probabilities of sequences, and, naively, this sum has an exponential number of terms.  However, when probabilities decompose into products of score functions of adjacent states, the forward-backward algorithm makes the process efficient.  The extra $\gamma(s)$ term makes our sampling process more complex in a way that is similar to semi-Markov models~\cite{Y10}. In graphical model terms, $\gamma$ can be thought of as a prior that discourages repeats: $p_{\textrm{prior}}(s) \propto 1/\gamma(s)$.  This prior can also be written as a product of $n$ Poisson distributions, each with parameter $\lambda = 1$.  
\section{Accuracy Intuition}
\label{sec:utility_guarantees}
\algo~applies the exponential mechanism once to output $m$ quantiles.  The closest competitor algorithms also apply the exponential mechanism but use $m$ invocations to produce $m$ quantiles. To build intuition for why the former approach achieves better utility, we recall the standard accuracy guarantee for the exponential mechanism:
\begin{lemma}[\citet{MT07}]
\label{lem:exp_guarantee}
    Let $M$ be an $\eps$-DP instance of the exponential mechanism having score function $u$ with sensitivity $\Delta_u$ and output space $\Y$. Then for database $X$, with probability at least $1-\beta$, $M$ produces output $y$ such that
    $$u(X,y) \geq \max_{y^* \in \Y} u(X,y^*) - \frac{2\Delta_u \log(|\Y|/\beta)}{\eps}.$$
\end{lemma}
\vspace{-5pt}
For simplicity, suppose we have uniform data where all interval widths $x_{i+1} - x_i$ are identical. As shown by the experiments in the next section, this is not necessary for \algo~to obtain good utility, but we assume it for easier intuition. Then (modulo the minor term $\gamma(s)$ that accounts for rare repeated intervals in the output), $\P{M_{Q'}}{s} \propto \exp\left(\frac{\eps \cdot u_{Q'}(X, s)}{2\Delta_{u_Q}}\right)$.  This means that \algo's process of sampling intervals draws from a distribution whose shape is identical to an exponential mechanism with utility function $u_{Q'}$, but mismatched sensitivity term $\Delta_{u_Q} = 2$.  Since the proof of  \cref{lem:exp_guarantee} does not rely on the utility function matching the sensitivity term, we can still apply it to determine the accuracy of this interval sampling procedure.  The output space $\Y$ for \algo's interval-sampling has size $|S_\nearrow| \leq n^m$, so we expect to sample intervals yielding quantiles that in total misclassify $O(m\log(n)/\eps)$ points.

In contrast, $m$ invocations of a single-quantile exponential mechanism requires each invocation to satisfy roughly $\eps_i = \eps/\sqrt{m}$-DP (advanced composition). Because each invocation uses an output space of size $O(n)$, the total error guarantee via Lemma~\ref{lem:exp_guarantee} scales like $O(m\log(n)/\eps_i)$. Since $m/\eps_i = \omega(m)/\eps$ for even the best known composition bounds for the exponential mechanism~\cite{DDR20}, these approaches incur error with a superlinear dependence on $m$. This contrasts with \algo's error, which has only a linear dependence on $m$.

\section{Experiments}
\label{sec:experiments}
We now empirically evaluate \algo~against three alternatives: \apq, \csmooth, and \aggtree. Discussion of some omitted alternatives appears in \cref{sec:bad_algos}. All experiment code is publicly available~\cite{Go21}.

\subsection{Comparison Algorithms}
\label{subsec:algos}
\paragraph{\apq :} Our first comparison algorithm \apq~uses independent applications of the exponential mechanism. \citet{S11} introduced the basic \pq~algorithm for estimating one quantile, and it has since been incorporated into the SmartNoise~\cite{S20} and IBM~\cite{I19} differential privacy libraries. \pq~thus gives us a meaningful baseline for a real-world approach. \pq~uses the exponential mechanism to estimate a single quantile $q$ via the utility function
$u(X,o) = ||\{x \in X \mid x \leq o\}| - qn|$.
\begin{lemma}
\label{lem:pq_sensitivity}
$u$ defined above has $L_1$ sensitivity $\Delta_{u} = 1$.
\end{lemma}
\vspace{-5pt}
\begin{proof}
    Consider swapping $x \in X$ for $x'$, and fix some $o$. If $x, x' \leq o$ or $x, x' > o$, then $u(X,o) = u(X',o)$. If exactly one of $x$ or $x'$ is $\leq o$, then $|u(X,o) - u(X',o)| = 1$.
\end{proof}
\vspace{-5pt}
\pq~takes user-provided data bounds $a$ and $b$ and runs on $X^+ = X \cup \{a,b\}$. After sorting $X^+$ into intervals of adjacent data points $I_0, \ldots, I_n$, \pq~selects an interval $I_j$ with probability proportional to
$$\mathrm{score}(X,I_j) = \exp(-\eps|j - qn|/2) \cdot |I_j|$$
and randomly samples the final quantile estimate from $I_j$.

To estimate $m$ quantiles with \apq, we call \pq~$m$ times with $\eps$ computed using the exponential mechanism's nonadaptive composition guarantee~\cite{DDR20}. Details appear in \cref{sec:app_comparisons}, but we note that this is the tightest known composition analysis for the exponential mechanism. Since our experiments use $n=1000$ data points, we always use $\delta = 10^{-6}$ in accordance with the recommendation that $\delta \ll \tfrac{1}{n}$ (see the discussions around the definition of differential privacy from~\citet{DR14} and~\citet{V17}).

\narxiv{\begin{figure*}[!htbp]
\begin{center}
    \includegraphics[scale=0.5]{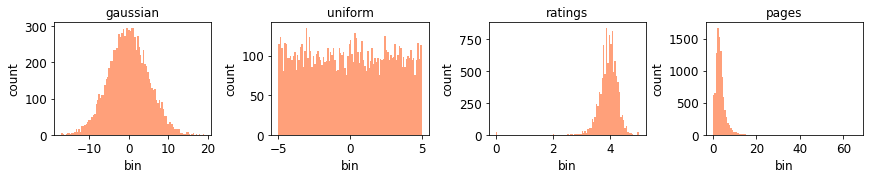}
    \vspace{-10pt}
    \caption{Histograms for the four datasets. Each plot uses 10,000 random samples, and the data range is divided into 100 equal-width bins.}
    \label{fig:data_histograms}
\end{center}
\end{figure*}}

\paragraph{\csmooth :} Our second comparison algorithm is \csmooth, which combines the \emph{smooth sensitivity}  framework introduced by \narxiv{\citet{NRS07}}\arxiv{Nissim, Raskhodnikova, and Smith~\cite{NRS07}} with concentrated differential privacy (CDP)~\cite{DR16, BS16}. The basic idea of smooth sensitivity  is to circumvent global sensitivity by instead using a smooth analogue of local sensitivity. This is useful for problems where the global sensitivity is large only for ``bad'' datasets.
\begin{definition}
\label{def:sensitivities}
    For function $f \colon X^n \to \mathbb{R}$, the \emph{local sensitivity} $\Delta_f(X)$ of $f$ for dataset $X$ is $\max_{X' \mid X \sim X'} |f(X) - f(X')|$.
\end{definition}
Recall that global sensitivity is defined over all possible pairs of datasets. In contrast, local sensitivity is also parameterized by a fixed dataset $X$ and defined only over neighbors of $X$. It is therefore possible that $\Delta_f(X) \ll \Delta_f$. For example, if $M$ is the median function and we set $X = [-100, 100]$, then $\Delta_M(\{-1,0,1\}) = 1$ while $\Delta_M(\{-100,0,100\}) = 100$. However, this also shows that local sensitivity itself reveals information about the dataset. The insight of~\citet{NRS07} is that it is possible to achieve differential privacy and take advantage of lower local sensitivity by adding noise calibrated to a ``smooth'' approximation of $\Delta_f(X)$.
\begin{definition}[\citet{NRS07}]
    For $t > 0$, the \emph{$t$-smooth sensitivity} of $f$ on database $X$ of $n$ points is
    $$\S_f^t(X) = \max_{X' \in \X^n} e^{-t \cdot d(X,X')} \cdot \Delta_f(X').$$
\end{definition}
\vspace{-5pt}
Details for computing the median's smooth sensitivity appear in \cref{sec:app_comparisons}. We now turn to the CDP portion of \csmooth. CDP is a variant of differential privacy that offers comparable privacy guarantees with often tighter privacy analyses.~\citet{BS19} showed how to combine CDP with the smooth sensitivity framework. Our experiments use the Laplace Log-Normal noise distribution, which achieved the best accuracy results in the experiments of~\citet{BS19}.

One complication of \csmooth~is the need to select several parameters to specify the noise distribution. We tuned these parameters on data from $N(0,1)$ to give \csmooth~the strongest utility possible without granting it distribution-specific advantages (see ~\cref{sec:app_comparisons}). To compare \algo 's pure DP guarantee to \csmooth 's CDP guarantee, we use the following lemma:

\begin{lemma}[Proposition 1.4~\cite{BS16}]
\label{lem:pure_to_concentrated}
    If an algorithm is $\eps$-DP, then it is also $\tfrac{\eps^2}{2}$-CDP.
\end{lemma}

We thus evaluate our $\eps$-DP algorithm \algo~against an $\tfrac{\eps^2}{2}$-CDP \csmooth. This comparison favors \csmooth: recalling our requirement that approximate DP algorithms have $\delta \leq 10^{-6}$, the best known generic conversion from CDP to approximate DP only says that a $\tfrac{1}{2}$-CDP algorithm is $(\eps,10^{-6})$-DP for $\eps \geq 5.76$ (Proposition 1.3,~\cite{BS16}). A more detailed discussion of DP and CDP appears in Section 4 of the work of \narxiv{\citet{CKS20}}\arxiv{Canonne, Kamath, and Steinke~\cite{CKS20}}.

As with \apq, to estimate $m$ quantiles with \csmooth, we call it $m$ times with an appropriately reduced privacy parameter. This time, we use CDP's composition guarantee:
\vspace{-10pt}
\begin{lemma}[Proposition 1.7~\cite{BS16}]
\label{lem:cdp_composition}
    The composition of $k$ $\rho$-CDP algorithms is $k\rho$-CDP.
\end{lemma}
From \cref{lem:pure_to_concentrated} the overall desired privacy guarantee is $\tfrac{\eps^2}{2}$-CDP, so we use $\eps' = \tfrac{\eps}{\sqrt{m}}$ in each call.

\paragraph{\aggtree:} The final comparison algorithm, \aggtree, implements the tree-based counting algorithm~\cite{DNPRY10, CSS11} for CDF estimation. This $\eps$-DP algorithm produces a data structure that yields arbitrarily many quantile estimates. Informally, \aggtree~splits the data domain into buckets and then builds a tree with branching factor $b$ and height $h$ where each leaf corresponds to a bucket. Each node of the tree has a count, and each data point increments the count of $h$ nodes. It therefore suffices to initialize each node with $\lap{h/\eps}$ noise to guarantee $\eps$-DP for the overall data structure, and the data structure now supports arbitrary range count queries. A more detailed exposition appears in the work of~\citet{KU20}. As with \csmooth, our experiments tune the hyperparameters $b$ and $h$ on $N(0,1)$ data. We also use the aggregation technique described by~\citet{H15}, which combines counts at different nodes to produce more accurate estimates.

\subsection{Data Description}
\label{subsec:data}
We evaluate our four algorithms on four datasets: synthetic Gaussian data from $N(0,5)$, synthetic uniform data from $U(-5,5)$, and real collections of book ratings and page counts from Goodreads~\cite{S19}  (Figure~\ref{fig:data_histograms}).

\arxiv{
\begin{figure*}[!htbp]
\begin{center}
    \includegraphics[scale=0.5]{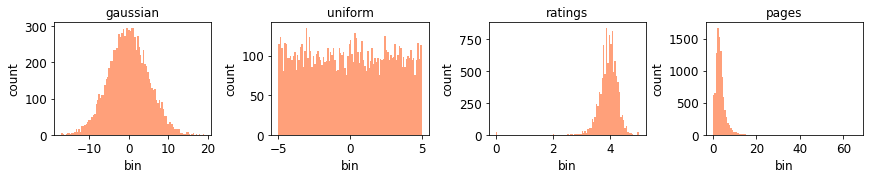}
    \caption{Histograms for the four datasets. Each plot uses 10,000 random samples, and the data range is divided into 100 equal-width bins.}
    \label{fig:data_histograms}
\end{center}
\end{figure*}
\vspace{-20pt}}

\subsection{Accuracy Experiments}
\label{subsec:experiment_acc}
\narxiv{
\begin{figure*}[!htbp]
\begin{center}
    \includegraphics[scale=0.42]{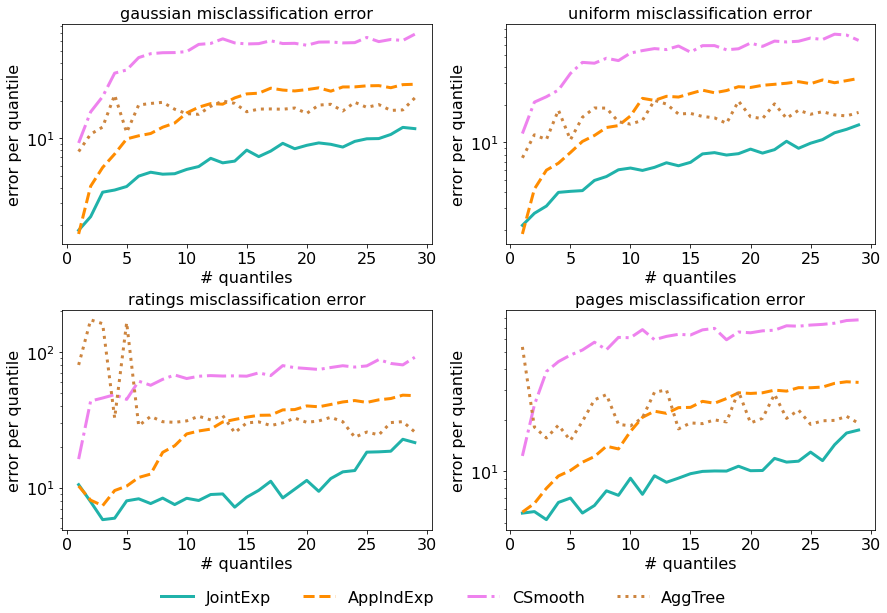}
\end{center}
\vspace{-15pt}
\caption{Average \# misclassified points per quantile vs. \# quantiles, averaged over 50 trials with $\eps = 1$. Note the logarithmic $y$-axis.}
\label{fig:error_experiments}
\end{figure*}}
\arxiv{
\begin{figure*}[!htbp]
\begin{center}
    \includegraphics[scale=0.42]{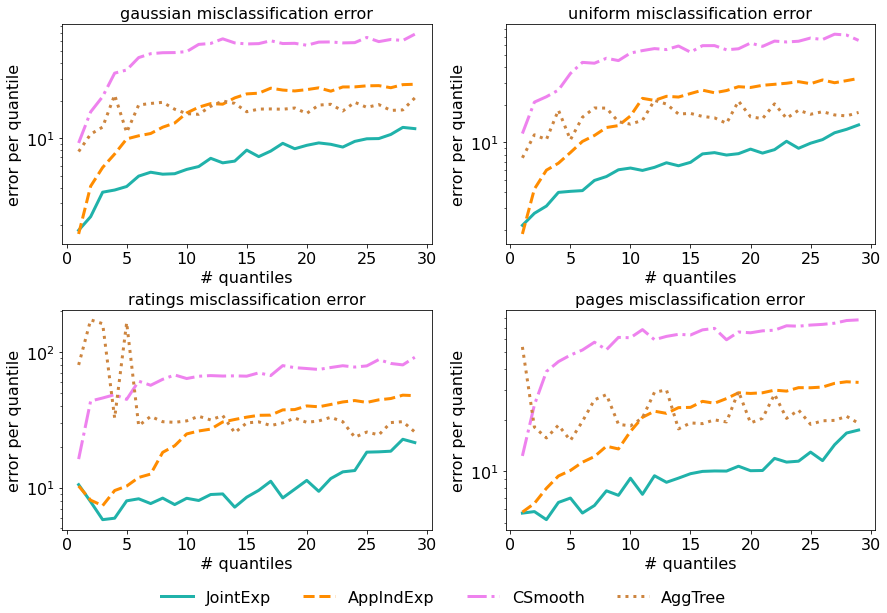}
\end{center}
\vspace{-15pt}
\caption{Average \# misclassified points per quantile vs. \# quantiles, averaged over 50 trials with $\eps = 1$. Note the logarithmic $y$-axis.}
\label{fig:error_experiments}
\end{figure*}}

Our error metric is the number of ``missed points'': for each desired quantile $q_j$, we take the true quantile estimate $o_j$ and the private estimate $\hat o_j$, compute the number of data points between $o_j$ and $\hat o_j$, and sum these counts across all $m$ quantiles. For each dataset, we compare the number of missed points for all five algorithms as $m$ grows. Additional plots for distance error appear in \cref{sec:app_comparisons}, but we note here that the trends are largely the same.

In each case, the requested quantiles are evenly spaced. $m=1$ is median estimation, $m=2$ requires estimating the 33rd and 67th percentiles, and so on.  We average scores across 20 trials of 1000 random samples. For every experiment, we take $[-100,100]$ as the (loose) user-provided data range. For the Goodreads page numbers dataset, we also divide each value by 100 to scale the values to $[-100,100]$. Experiments for $\eps=1$ appear in Figure~\ref{fig:error_experiments}. 

Across all datasets, a clear effect appears: for a wide range of the number of quantiles $m$, \algo~dominates all other algorithms. At $m=1$, \algo~matches \apq~and obtains roughly an order of magnitude better error than \csmooth~or \aggtree. As $m$ grows, \algo~consistently obtains average quantile error roughly 2-3 times smaller than the closest competitor, until the gap closes around $m=30$. \algo~thus offers both the strongest privacy guarantee and the highest utility for estimating any number of quantiles between $m=1$ and approximately $m=30$.

\subsection{Time Experiments}
\label{subsec:experiment_time}
\arxiv{
\begin{figure}
\begin{center}
\includegraphics[scale=0.5]{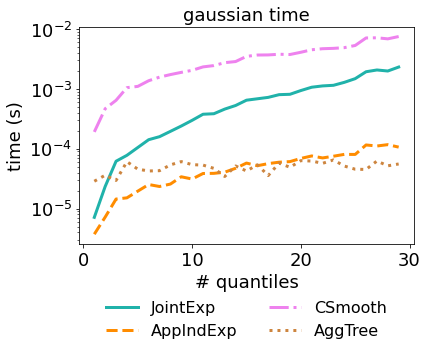}
\end{center}
\caption{Time vs \# quantiles $m$, $\eps=1$, averaged across 50 trials of 1,000 samples.}
\label{fig:arxiv_time}
\end{figure}}
We conclude by evaluating the methods by runtime. The number of data points and quantiles are the main determinants of time, so we only include time experiments using Gaussian data. All experiments were run on a machine with two CPU cores and 100GB RAM. As seen in \arxiv{\cref{fig:arxiv_time}}\narxiv{\cref{fig:narxiv_time}}, \algo~has time performance roughly in between that of the slowest algorithm, \csmooth, and \apq~or \aggtree. For estimating $m=30$ quantiles, \algo~takes roughly 1 ms for $n=1,000$ points and slightly under 1 minute for $n=1$ million points.
\narxiv{
\begin{figure}[!htbp]
    \begin{center}
        \includegraphics[scale=0.43]{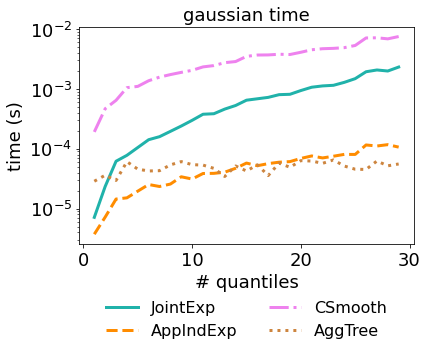}
    \end{center}
    \vspace{-15pt}
\caption{Time vs \# quantiles $m$ for $\eps=1$, averaged across 50 trials of 1,000 samples.}
\label{fig:narxiv_time}
\end{figure}}
\section{Future Directions}
\label{sec:conc}
In this work we constructed a low-sensitivity exponential mechanism for differentially private quantile estimation and designed a dynamic program to sample from it efficiently. The result is a practical algorithm that achieves much better accuracy than existing methods. A possible direction for future work is exploring other applications of the exponential mechanism where the utility function is low sensitivity and can be decomposed into ``local'' score functions, as in the pairwise interval terms of $\phi$. More precisely, by analogy to the graphical models techniques generally known as belief propagation~\cite{P82}, any utility function whose outputs have a chain or tree dependency structure should be tractable to sample.
\section{Acknowledgements}
We thank Thomas Steinke for discussions of concentrated differential privacy; Andrés Muñoz Medina for comments on an early draft of this paper; Uri Stemmer for discussion of the threshold release literature; and Peter Kairouz and Abhradeep Guha Thakurta for discussion of the aggregated tree mechanism.

\bibliography{references}
\bibliographystyle{plainnat}

\appendix
\onecolumn
\section{Full Proofs}
\label{subsec:proofs}
We start with the add-remove version of the sensitivity analysis for $\Delta_{u_Q}$. We proved the swap version as Lemma~\ref{lem:algo_sensitivity} in the main body, and this was the focus of the paper. The (slightly more favorable) add-remove version appears below for completeness.
\begin{lemma}
    In the add-remove model, $\Delta_{u_Q} = 2[1 - \min_{j \in  [m+1]}(q_j - q_{j-1})].$
\end{lemma}
\begin{proof}
    Consider neighboring databases $X' = X \cup \{x'\}$ where $|X'| = n' = n+1$ and $|X| = n$. Let $n'(\cdot, \cdot)$ denote an interval count using $X'$, and let $n(\cdot, \cdot)$ denote an interval count using $X$. All data points are clipped to $[a,b]$, $o_0 = a$, and $o_{m+1} = b+1$, so there exists some $[o_{j^*-1}, o_{j^*})$ containing $x'$. $o$ is nondecreasing and these intervals are half-open, so these intervals do not intersect. Thus, there is exactly one $[o_{j^*-1}, o_{j^*})$ containing $x'$. Then for $j \neq j^*$, $n'(o_{j-1}, o_j) = n(o_{j-1}, o_j)$ and $(q_{j} - q_{j-1})n' - (q_{j} - q_{j-1})n = q_{j} - q_{j-1}$. Thus
    $$u_Q(X, o) = -|n(o_{j^*-1}, o_{j^*}) - n_{j^*}| - \sum_{j \neq j^*} |n(o_{j-1}, o_j) - n_j|$$
    and
    $$u_Q(X', o) = -|n(o_{j^*-1}, o_{j^*}) + 1 - (q_j^* - q_{j^*-1})(n+1)| - \sum_{j \neq j^*} |n(o_{j-1}, o_j) - (q_j - q_{j-1})(n+1)|.$$
    The distance between $u_Q(X, o)$ and $u_Q(X', o)$ contributed by the first term is
    $$||n(o_{j^*-1}, o_{j^*}) - (q_{j^*} - q_{j^*-1})n| - |n(o_{j^*-1}, o_{j^*}) + 1 - (q_{j^*} - q_{j^*-1})(n+1)|| = 1 - (q_{j^*} - q_{j^*-1}),$$
    and the distance contributed by the second term is
    $$\sum_{j \neq j^*} ||n(o_{j-1}, o_j) - (q_j - q_{j-1})n| - |n(o_{j-1}, o_j) - (q_j - q_{j-1})(n + 1)|| \leq \sum_{j \neq j^*} (q_j - q_{j-1}).$$
    Thus
    \begin{align*}
        |u_Q(X, o) - u_Q(X', o)| \leq&\ 1 - (q_{j^*} - q_{j^*-1}) + \sum_{j \neq j^*} (q_j - q_{j-1}) \\
        =&\ 2[1 - (q_{j^*} - q_{j^*-1})].
    \end{align*}
    The last equality follows from the fact that the sum over all quantile gaps is $1$, so the sum over all but the $q_{j^*} - q_{j^* - 1}$ gap is $1 - (q_{j^*} - q_{j^* - 1})$.  The quantity
    $2[1 - (q_{j^*} - q_{j^*-1})]$ is maximized by minimizing $(q_{j^*} - q_{j^*-1})$, which gives the final sensitivity bound.
\end{proof}
Next, we verify that the finite sampling improvement from \cref{subsec:finite} still samples from the correct distribution.
\MQPrime*
\begin{proof}
    Recall that the output space for $M_Q$ was $O_\nearrow = \{o = (o_1, \ldots, o_m) \mid a \leq o_1 \leq \cdots \leq o_m \leq b\}$. Define function $h$ on $[a,b]$ by $h(y) = |\{x \in X \mid x < y\}|$. Then $n(u,v) = h(v) - h(u)$, $h(o_0) = h(a) = 0 = i_0$, and $h(o_{m+1}) = h(b+1) = n = i_{m+1}$. Thus
    \begin{align*}
        u_Q(X,o) =&\ -\sum_{j \in [m+1]} |n(o_{j-1}, o_j) - n_j| \\
        =&\ -\sum_{j \in [m+1]} | h(o_j) - h(o_{j-1}) - n_j| \\
        =&\ u_{Q'}(X, (h(o_1), \ldots, h(o_m))).
    \end{align*}
    Therefore the normalization term for the distribution defined in \cref{eq:Q_density} is
    \begin{align}
        Z_Q =&\ \int_{O_\nearrow} \exp\left(\frac{\eps}{2\Delta_{u_Q}} \cdot u_Q(X, o)\right) do \nonumber \\
        =&\ \int_{O_\nearrow} \exp\left(\frac{\eps}{2\Delta_{u_Q}} \cdot u_{Q'}(X, (h(o_1), \ldots, h(o_m))\right) do. \label{eq:Z_Q}
    \end{align}
    Note that each $o = (o_1, \ldots, o_m) \in O_\nearrow$ has $o_1 \in [x_{i_1}, x_{i_1 + 1}), \ldots, o_m \in [x_{i_m}, x_{i_m + 1})$ for exactly one $s = (i_1, \ldots, i_m) \in S_\nearrow$. Shorthand this by $o \in s$, and let $g_O(s) = \{o \in O_\nearrow \mid o \in s\}$. Then
    \begin{align}
        (\ref{eq:Z_Q}) =&\ \sum_{s \in S_\nearrow} \int_{g_O(s)} \exp\left(\frac{\eps}{2\Delta_{u_Q}} \cdot u_{Q'}(X, (h(o_1), \ldots, h(o_m))\right) do \nonumber \\
        =&\ \sum_{s \in S_\nearrow} \int_{g_O(s)} \exp\left(\frac{\eps}{2\Delta_{u_Q}} \cdot u_{Q'}(X, s)\right) do \nonumber \\
        =&\ \sum_{s \in S_\nearrow} \exp\left(\frac{\eps}{2\Delta_{u_Q}} \cdot u_{Q'}(X, s)\right) \cdot \int_{g_O(s)} do \label{eq:Z_Q_2}.
    \end{align}
    We focus on the $\int_{g_O(s)} do$ term. If $h(o_1), \ldots, h(o_m)$ are all distinct, i.e. $o_1, \ldots, o_m$ come from distinct intervals between data points, then
    $$\int_{g_O(s)} do = \prod_{j=1}^m (x_{i_j + 1} - x_{i_j}).$$
    The remaining (and more complex) case is when $h(o_1), \ldots, h(o_m)$ are not distinct. Suppose $h(o_1), \ldots, h(o_k)$ are not distinct but the remaining $h(o_{k+1}), \ldots, h(o_m)$ are distinct and different from $h(o_1)$. Note that the non-distinct elements are consecutive since $o_1 \leq \cdots \leq o_m$. Then there is some $i \in I$ such that $o_1, \ldots, o_k \in [x_i, x_{i+1})$. Thus the set of valid $o_1, \ldots, o_k$ is exactly $\{(o_1, \ldots, o_k) \mid x_i \leq o_1 \leq \cdots \leq o_k < x_{i+1}\}$.
    
    We need to determine the volume of this set.  First, note that the collection consisting of \emph{all} sets of $k$ values from interval $i$ has volume $(x_{i+1} - x_i)^k$.  Then, note that the probability that $k$ values selected at random from an interval will be perfectly sorted is $1/k!$; this is the volume of the standard $k$-simplex, which is the set $\{(x_1, \ldots, x_k) \mid 0 \leq x_1 \leq \cdots \leq x_k \leq 1\}$.  Hence, for the set that we are interested in, $\{(o_1, \ldots, o_k) \mid x_i \leq o_1 \leq \cdots \leq o_k < x_{i+1}\}$, the volume is $\tfrac{(x_{i+1} - x_i)^k}{k!}$.
    
    More generally, this leads us to define the scaling factor $\gamma$ in \cref{subsec:finite}:
    $$\gamma(s) = \prod_{i \in I} \num{i}{s}!$$
    where $\num{i}{s}$ is the number of times $i$ appears in $s$, and we take $0! = 1$. $\gamma$ thus repeats the above scaling process for each interval according to its number of repetitions in $h(o_1), \ldots, h(o_m)$. It follows that for any $s \in S_\nearrow$,
    $$\int_{g_O(s)} do = \frac{\prod_{j=1}^m (x_{i_j + 1} - x_{i_j})}{\gamma(s)}.$$
    Returning to our original chain of equalities, we get
    \begin{align*}
        (\ref{eq:Z_Q_2}) =&\ \sum_{s \in S_\nearrow} \exp\left(\frac{\eps}{2\Delta_{u_Q}} \cdot u_{Q'}(X, s)\right) \cdot \frac{\prod_{j=1}^m (x_{i_j + 1} - x_{i_j})}{\gamma(s)} \\
        =&\ Z_{Q'}.
    \end{align*}
    Turning to the output density $f_Q$ for $M_Q$, by above
    $$f_Q(o) = \frac{1}{Z_{Q'}} \cdot \exp\left(\frac{\eps}{2\Delta_{u_Q}} \cdot u_Q(X, o)\right).$$
    For any $s \in S_\nearrow$ and any $o, o' \in g_O(s)$ we have $f_Q(o) = f_Q(o')$ and
    \begin{align*}
        \P{M_Q}{M_Q(X) \in g_O(s)} =&\ \frac{1}{Z_{Q'}} \cdot \int_{g_O(s)} \exp\left(\frac{\eps}{2\Delta_{u_Q}} \cdot u_{Q'}(X, (h(o_1), \ldots, h(o_m)))\right) do \\
        =&\ \frac{1}{Z_{Q'}} \cdot \exp\left(\frac{\eps}{2\Delta_{u_Q}} \cdot u_{Q'}(X, s)\right) \cdot \frac{\prod_{j=1}^m (x_{i_j + 1} - x_{i_j})}{\gamma(s)} \\
        =&\ \P{M_{Q'}}{M_{Q'}(X) \in g_O(s)}.
    \end{align*}
    $u_Q$ is constant over $o \in g_O(s)$ for any $s$, so conditioned on selecting a given $s = (i_1, \ldots, i_m) \in S_\nearrow$, $M_Q$ has a uniform output distribution over increasing sequences from $g_O(s)$, i.e.
    \begin{align*}
        f_Q(o) =&\ f_Q(o \mid o \in g_O(s)) \cdot \P{M_Q}{M_Q(X) \in g_O(s)} \\
        =&\ \prod_{j \in [m]} \frac{\gamma(s)}{x_{i_j+1} - x_{i_j}} \cdot \P{M_Q}{M_Q(X) \in g_O(s)} \\
        =&\ \prod_{j \in [m]} \frac{\gamma(s)}{x_{i_j+1} - x_{i_j}} \cdot \P{M_{Q'}}{M_{Q'}(X) \in g_O(s)} \\
        =&\ f_{Q'}(o)
    \end{align*}
    where the second and fourth equalities use $\int_{g_O(s)} do = \tfrac{\prod_{j=1}^m (x_{i_j + 1} - x_{i_j})}{\gamma(s)}$. Thus $M_Q$ and $M_{Q'}$ have identical output distributions.
\end{proof}

We now repeat this process for the efficient sampling improvement from \cref{subsec:fb}.
\Algo*
\begin{proof}
    We first verify that \algo~samples from $M_{Q'}$, which implies differential privacy. Since the uniform sampling step is unchanged, it suffices to show that the distribution over sampled sequences of intervals is correct.
    
    Let $S_\nearrow(j,i,k) = \{(i_1, \dots, i_j) \mid i_1 \leq \cdots \leq i_{j-k} < i_{j-k+1} = \cdots = i_j = i\}$ denote the set of nondecreasing sequences of length $j$ where exactly the last $k$ intervals are equal to $i$. We will first show that, for all $j \in [m]$, all $i \in I$, and all $k \in [j]$,
    \begin{equation}
    \label{eq:alpha}
        \alpha(j,i,k) = \sum_{\substack{s = (i_1, \dots, i_j) \\ s \in S_\nearrow(j,i,k)}} 
                \frac{1}{\gamma(s)} \prod_{j' \in [j]} \phi(i_{j'-1}, i_{j'}, j') \tau(i_{j'}).
    \end{equation}
    The base case of $j=1$ holds by definition, since we let $\alpha(1, i, 1) = \phi(0, i, 1) \tau(i)$ and $\gamma(s) = 1$ when $s$ is a sequence of length one. Before the induction step, we make our optimization for computing $\alpha$ explicit. First, we define
    \[
        \hat\alpha(j-1, \cdot) = \sum_{k < j} \alpha(j-1, \cdot, k),
    \]
    a vector of length $n+1$ that can be computed in time $O(mn)$. Then
    \[
        \alpha(j, i, 1) = \tau(i) \sum_{i' < i} \phi(i', i, j) \sum_{k < j} \alpha(j-1, i', k)
        = \tau(i) \sum_{i' < i} \phi(i', i, j) \hat \alpha(j-1, i').
    \]
    Each $\alpha(j, i, 1)$ sums $O(n)$ terms, so a straightforward computation of $\alpha(j, \cdot, 1)$ in its entirety takes time $O(n^2)$. However, we can improve on this by noticing that, after fixing $j$, each $\phi(i, i', j)$ depends only on $i' - i$. $\phi(\cdot, \cdot, j)$ is therefore a Toeplitz matrix, i.e. a matrix with constant diagonals:
    \begin{equation}
    \label{eq:matrix_phi}
        \phi(\cdot, \cdot, j) =
        \begin{bmatrix}
            \exp\left(-\frac{\eps n_j}{2\Delta_{u_Q}}\right) & \exp\left(-\frac{\eps|1 - n_j|}{2\Delta_{u_Q}}\right) & \exp\left(-\frac{\eps|2 - n_j|}{2\Delta_{u_Q}}\right) & \dots  & \exp\left(-\frac{\eps|n - n_j|}{2\Delta_{u_Q}}\right) \\
            0 & \exp\left(-\frac{\eps n_j}{2\Delta_{u_Q}}\right) & \exp\left(-\frac{\eps|1 - n_j|}{2\Delta_{u_Q}}\right) & \cdots  & \exp\left(-\frac{\eps|n - 1 - n_j|}{2\Delta_{u_Q}}\right) \\
            \vdots & \vdots & \ddots & \cdots & \vdots \\
            0 & 0 & 0 & \cdots & \exp\left(-\frac{\eps|1 - n_j|}{2\Delta_{u_Q}}\right) \\
            0 & 0 & 0 & \cdots  & \exp\left(-\frac{\eps n_j}{2\Delta_{u_Q}}\right)
        \end{bmatrix}.
    \end{equation}
    It follows that we can use the Fast Fourier Transform (FFT) to multiply $\phi(\cdot,\cdot,j)$ by a vector of length $n+1$ in time $O(n\log(n))$ instead of the typical $O(n^2)$ (a brief reference for this fact appears in the following lecture notes~\cite{B19}). Letting $\times$ denote element-wise product, we now rewrite  
    \[
        \alpha(j, \cdot, 1) = \tau(\cdot) \times \left(\phi(\cdot, \cdot, j)_0^T\hat \alpha(j-1, \cdot)^T\right),
    \] 
    where $\phi(\cdot,\cdot,j)$ denotes $\phi$ with the diagonal set to 0,
    and use the FFT to compute the second term in time $O(n\log(n))$, since $\phi(\cdot, \cdot, j)^T$ is also Toeplitz. It therefore takes overall time $O(mn\log(n) + m^2n)$ to repeat this for each $j$ and compute $\alpha(\cdot, \cdot, 1)$.
    
    Returning to the inductive step, suppose \cref{eq:alpha} holds for $j' < j$. Then
    \begin{align*}
        \alpha(j, i, 1) 
        &= \tau(i) (\phi(\cdot, \cdot, j)^T \hat \alpha(j-1, \cdot)^T)_i \\
        &= \sum_{i' < i} \tau(i) \phi(i', i, j) \hat \alpha(j-1, i') \\
        &= \sum_{i' < i} \tau(i) \phi(i', i, j) \sum_{k < j} \alpha(j-1, i', k)\\
        &= \sum_{i' < i} \tau(i) \phi(i', i, j) \sum_{k < j} 
            \sum_{\substack{s' = (i_1, \dots, i_{j-1}) \\ s' \in S_\nearrow(j-1,i',k)}} 
                \frac{1}{\gamma(s')} \prod_{j' \in [j-1]} \phi(i_{j'-1}, i_{j'}, j') \tau(i_{j'}) \\
        &= \sum_{i' < i} \sum_{k < j} 
            \sum_{\substack{s' = (i_1, \dots, i_{j-1}) \\ s' \in S_\nearrow(j-1,i',k)}} 
                \tau(i) \phi(i', i, j) \frac{1}{\gamma(s')} \prod_{j' \in [j-1]} \phi(i_{j'-1}, i_{j'}, j') \tau(i_{j'})\\
        &= \sum_{\substack{s = (i_1, \dots, i_j) \\ s \in S_\nearrow(j,i,1)}} 
            \frac{1}{\gamma(s)} \prod_{j' \in [j]} \phi(i_{j'-1}, i_{j'}, j') \tau(i_{j'}),
    \end{align*}
    since every sequence in $S_\nearrow(j,i,1)$ consists of a sequence in $S_\nearrow(j-1, i', k)$ for some $i' < i$ and $k < j$ with an $i$ appended to the end. Note that the appending of only a single $i$ means that $\gamma(s) = \gamma(s')$. Similarly,
    \begin{align*}
    \alpha(j, i, k > 1) 
    &= \tau(i) \cdot \phi(i, i, j) \cdot \alpha(j-1, i, k-1) / k\\
    &= \frac{\tau(i)}{k} \phi(i, i, j) \sum_{\substack{s' = (i_1, \dots, i_{j-1}) \\ s' \in S_\nearrow(j-1,i,k-1)}} 
                \frac{1}{\gamma(s')} \prod_{j' \in [j-1]} \phi(i_{j'-1}, i_{j'}, j') \tau(i_{j'}) \\
    &= \sum_{\substack{s = (i_1, \dots, i_j) \\ s \in S_\nearrow(j,i,k)}} 
                \frac{1}{\gamma(s)} \prod_{j' \in [j]} \phi(i_{j'-1}, i_{j'}, j') \tau(i_{j'}),
    \end{align*}
    since every sequence in $S_\nearrow(j,i,k>1)$ consists of a sequence in $S_\nearrow(j-1, i, k-1)$ with an $i$ appended to the end and, when $i$ appears $k-1$ times in sequence $s'$ and $s$ is equal to $s'$ with an $i$ appended to the end, $\gamma(s) = k \gamma(s')$. Thus \cref{eq:alpha} holds, and we have the ``forward'' step: $\alpha(j,i,k)$ is the (unnormalized) mass of nondecreasing length-$j$ sequences ending in $k$ repetitions of $i$.
    
    Now consider the backward sampling process in \algo. In the first step, we sample a pair 
    \begin{align*}
        (i,k) &\propto \alpha(m, i, k) \phi(i, n, m+1) \\
        &= \left[\sum_{\substack{s = (i_1, \dots, i_m) \\ s \in S_\nearrow(m,i,k)}} 
                \frac{1}{\gamma(s)} \prod_{j' \in [m]} \phi(i_{j'-1}, i_{j'}, j') \tau(i_{j'})\right]\phi(i, n, m+1)\\
        &= \sum_{\substack{s = (i_1, \dots, i_m) \\ s \in S_\nearrow(m,i,k)}} 
                \frac{1}{\gamma(s)} \prod_{j' \in [m+1]} \phi(i_{j'-1}, i_{j'}, j') \prod_{j' \in [m]} \tau(i_{j'})\\
        &\propto \sum_{s \in S_\nearrow(m,i,k)} \P{M_{Q'}}{s}~,
    \end{align*}
    where the second equality uses the fact that we fix $i_{m+1} = n$. Since $\{S_\nearrow(m,i,k)\}_{i,k}$ is a partition of $S_\nearrow$ (that is, every sequence in $S_\nearrow$ appears in $S_\nearrow(m,i,k)$ for exactly one value of the pair $(i,k)$), we conclude that $(i,k)$ is sampled according to the marginal probability that a sequence sampled from $\Po{M_{Q'}}$ ends in exactly $k$ copies of $i$.
    
    Continuing the backward recursion, if the values of $s_\text{suffix} = (i_{j+1}, \dots, i_m)$ have already been sampled, then at the next step we sample a pair
    \begin{align*}
        (i < i_{j+1},k) &\propto \alpha(j, i, k) \phi(i, i_{j+1}, j+1) \\
        &= \left[\sum_{\substack{s_\text{prefix} = (i_1, \dots, i_j) \\ s_\text{prefix} \in S_\nearrow(j,i,k)}} 
                \frac{1}{\gamma(s_\text{prefix})} \prod_{j' \in [j]} \phi(i_{j'-1}, i_{j'}, j') \tau(i_{j'}) \right]\phi(i, i_{j+1}, j+1) \\
        &\propto \left[\sum_{\substack{s_\text{prefix} = (i_1, \dots, i_j) \\ s_\text{prefix} \in S_\nearrow(j,i,k)}} 
                \frac{1}{\gamma(s_\text{prefix})} \prod_{j' \in [j]} \phi(i_{j'-1}, i_{j'}, j') \tau(i_{j'}) \right] \phi(i, i_{j+1}, j+1) \\
                &\cdot \left[\frac{1}{\gamma(s_\text{suffix})} \prod_{j' = j+2}^{m+1} \phi(i_{j'-1}, i_{j'}, j') \prod_{j'=j+1}^m \tau(i_{j'})\right]\\
        &= \sum_{\substack{s_\text{prefix} = (i_1, \dots, i_j) \\ s_\text{prefix} \in S_\nearrow(j,i,k)}} 
                \frac{1}{\gamma(s_\text{prefix})\gamma(s_\text{suffix})} \prod_{j' \in [m+1]} \phi(i_{j'-1}, i_{j'}, j') \prod_{j' \in [m]} \tau(i_{j'}) \\
        &\propto \sum_{\substack{s_\text{prefix} = (i_1, \dots, i_j) \\ s_\text{prefix} \in S_\nearrow(j,i,k)}} \P{M_{Q'}}{s_\text{prefix} + s_\text{suffix}},
    \end{align*}
    where $+$ denotes sequence concatenation, and we use the fact that, since $s_\text{prefix}$ and $s_\text{suffix}$ are nondecreasing and $i_j < i_{j+1}$, $\gamma(s_\text{prefix} + s_\text{suffix}) = \gamma(s_\text{prefix})\gamma(s_\text{suffix})$.
    Again, the set of nondecreasing sequences of length $j$ can be partitioned into disjoint subsets $\{S_\nearrow(j,i,k)\}_{i,k}$, thus the pair $(i,k)$ is sampled according to the marginal probability that a sequence sampled from $\Po{M_{Q'}}$, conditional on having the suffix $s_\text{suffix}$, has a $j$-length prefix ending in exactly $k$ copies of $i$.
    
    Inductively, then, \algo~samples a sequence $s$ according to $\Po{M_{Q'}}$. It remains to show that \cref{alg:algo} satisfies the claimed time and space guarantees.
    
    \paragraph{Time analysis.} The first for-loop in \cref{alg:algo} computes $\alpha(1, \cdot, 1)$ in total time $O(n)$. Each iteration of the second for-loop, over $j = 2, \ldots, m$, computes $\hat \alpha(j-1, \cdot)$ in time $O(mn)$, computes $\alpha(j, \cdot, 1)$ in time $O(n\log(n))$ using FFT multiplication, and finally spends $O(m)$ time setting $\alpha(j, i, \cdot)$. The second for-loop thus takes total time $O(m^2n + mn\log(n))$. Having computed $\alpha$, each sampling of $(i, k)$ takes time $O(mn)$, so the final sampling takes time $O(m^2n)$. Summing up, the total time is $O(mn\log(n) + m^2n)$. 
    
    \paragraph{Space analysis.} $\hat{\alpha}$ takes $O(mn)$ space, the FFT relying on the Toeplitz expression of $\phi(\cdot, \cdot, j)$ takes space $O(n)$, and $\alpha$ takes $O(m^2n)$ space. All other variables in the algorithm occupy a constant amount of space, so the overall space usage is $O(m^2n)$.
\end{proof}

\section{Logarithm Trick}
\label{subsec:log_trick}
In this section, we give details for a more numerically stable logarithmic version of \algo. Recall that we defined $\alpha(1, i, 1) = \phi(0, i, 1)\tau(1)$ and $\alpha(1, i, k) = 0$ for $k > 1$. The former becomes $\ln(\alpha(1, i, 1)) = \ln(\phi(0, i, 1)) + \ln(\tau(1))$ and the latter $\ln(\alpha(1, i, k)) = -\infty$, e.g. using \texttt{-numpy.inf} in Python.

We now turn to $\ln(\alpha(j, \cdot, \cdot))$ for $j = 2, \ldots, m$. To set
\[
    \ln(\hat \alpha(j-1, i)) = \ln\left(\sum_{k < j} \alpha(j-1, i, k)\right)
\]
$\ln(\alpha)$ terms that have already been computed, we use the following method for summing a vector of quantities $a$ given its component-wise logarithmic form $\ln(a)$
\begin{enumerate}
    \item Compute the maximum element in the vector: $M_a = \max(\ln(a))$.
    \item Component-wise subtract off the maximum element and exponentiate: $a = \exp(\ln(a) - M_a)$.
    \item Sum outside of logspace, then return to logspace: $c = \ln(\text{sum}(a))$.
    \item Add back the maximum: return $c + M_a$.
\end{enumerate}
An example implementation is \texttt{scipy.special.logsumexp}~\cite{Sc20}.

In the computation of different $\ln(\alpha(j, i, 1))$ using
\[
    \alpha(j, \cdot, 1) = \tau(\cdot) \times \left(\phi(\cdot, \cdot, j)^T\hat \alpha(j-1, \cdot)\right)
\] 
we want to multiply a Toeplitz matrix $A$ and vector $B$ given their component-wise logarithmic forms $\ln(A)$ and $\ln(B)$ by a similar process. Since $A$ is Toeplitz, we only need to work with its first column $\ln(A_c)$ and first row $\ln(A_r)$. Then we:
\begin{enumerate}
    \item Compute the maximum element in $A_c$ and $A_r$ and the maximum element in $\ln(B)$: $M_A = \max(\max(\ln(A_c)), \max(\ln(A_r)))$ and $M_B = \max(\ln(B))$.
    \item Component-wise subtract off the maximum element and exponentiate: $A_c = \exp(\ln(A_c) - M_A)$, $A_r = \exp(\ln(A_r) - M_A)$ and $B = \exp(\ln(B) - M_B)$.
    \item Do the FFT matrix-vector multiplication outside of logspace, then return to logspace: $C = \ln(A \times B)$.
    \item Add back the maxima: return $C + M_A + M_B$.
\end{enumerate}
An example implementation for non-FFT matrix multiplication can be found on StackOverflow~\cite{SO14}.

\section{Sampling by ``Racing'' Method}
\label{subsec:racing}
The ``racing'' method is originally due to Ilya Mironov. To the best of our knowledge, full exposition and proofs first appeared in the work of~\citet{MG20}. We recap their exposition here. The main tool is the following result:
\begin{lemma}[Proposition 5~\cite{MG20}]
\label{lem:racing}
    Let $U_1, \ldots, U_N \sim U(0,1)$ be uniform random samples from $[0,1]$ and define random variable $R = \arg \min_{[N]} \left[\ln(\ln(1/U_k)) - \ln(p_k)\right]$. Then $\P{R}{k} = \frac{p_k}{\sum_{j=1}^N p_j}$.
\end{lemma}
\cref{lem:racing} enables us to sample from distributions that depend on small probabilities $p_k$ by instead using their logarithms. In combination with the logarithm trick from \cref{subsec:log_trick}, we avoid dealing with exponentiated terms entirely.

\section{Discussion of Other Quantile Algorithms}
\label{sec:bad_algos}
In this section, we discuss the private quantile estimation algorithms of \citet{DL09} and \narxiv{\citet{TVZ20}}\arxiv{Tzamos, Vlatakis-Gkaragkounis, and Zadik~\cite{TVZ20}} and explain why we do not include them in our experiments.  Both of these are single quantile algorithms and would require $m$ compositions in order to estimate $m$ quantiles.

\citet{DL09} define a ``propose-test-release'' algorithm. Briefly, it discretizes the space into bins of equal width, then computes how many points in the dataset must change to move the true quantile out of its current bin. If this number is too small (specifically, if it is no larger than $\ln^2(n) + 2$, the ``test''), then the algorithm does not produce an answer. Otherwise, the answer is the true quantile plus Laplace noise whose scale is six times the bin width (``release'').

We can ballpark the accuracy of this method on the uniform data distribution used in our experiments, i.e. $n=1000$ samples from $U(-5,5)$. Then $\ln^2(n) + 2 \approx 50$. If we choose a bin width such that the bin with the true median contains 100 points, then it takes at most 50 swaps to move the median out of that bin. We must therefore choose, at a minimum, a bin size such that the bin containing the median contains at least 100 points. Even if we successfully make this choice, then the resulting output will \emph{still} be far less accurate than that of all the other methods tested in the experiments. This is because a successful choice requires a bin width $\geq 1$, so the algorithm releases the true median value of $\approx 0$ plus Laplace noise with scale 6. With that scaling, the estimated median is at one of the limits of the $[-5, 5]$ range with probability $\approx 0.434$. This means that the estimated median misclassifies roughly $500$ out of the $1000$ points over 40\% of the time, making its expected error in excess of $200$ points. For comparison, the algorithms that we test only require lower and upper bounds on the data (not knowledge of the distribution sufficient to choose a good bin width), always output an estimate, and produce average error $\leq 25$ for median estimation on uniform data.

We now turn to the private quantile estimation algorithm given by \citet{TVZ20}. This algorithm is also based on adding (a variant of) Laplace noise to the true median. The first drawback of this method is that its time complexity is $O(n^4)$ (see the footnote accompanying their definition of ``TypicalHamming''). This makes it impractical for datasets with more than a few hundred datapoints. The second drawback is the need to select several hyperparameters ($R, r, L, C$) to determine the specific Laplace noise distribution. While this hyperparameter selection does not affect the privacy guarantee, it does affect the utility. Their utility guarantees assume that the algorithm operator knows these distributional parameters a priori, but this assumption may be hard to satisfy in practice. In contrast, \algo~and \apq~only require the user to provide endpoints.

\section{Details For Comparison Algorithms}
\label{sec:app_comparisons}
\paragraph{\apq:} Privacy parameters for the $m$ invocations of the exponential mechanism come from the composition guarantee given by\arxiv{Dong, Durfee, and Rogers~\cite{DDR20}} \narxiv{\citet{DDR20}}. For simplicity, we give a less general (but not weaker) version of their result.
\begin{lemma}[Theorem 3~\cite{DDR20}]
\label{lem:apq_composition}
    Let mechanism $\A$ consist of $m$ nonadaptive $\eps$-DP applications of the exponential mechanism. Define
    $$t_\ell^* = \left[\frac{\eps_g + (\ell+1)\eps}{m+1}\right]_0^\eps \text{ and } p_{t_\ell^*} = \frac{e^{-t_\ell^*} - e^{-\eps}}{1 - e^{-\eps}}$$
    where $\left[x\right]_0^\eps$ denotes the value of $x$ clipped to interval $[0,\eps]$. Then $\A$ is $(\eps_g, \delta)$-DP for
    $$\delta = \max_{0 \leq l \leq m} \sum_{i=0}^m \left[\binom{m}{i} p_{t_\ell^*}^{m-i} \cdot (1 - p_{t_\ell^*})^i \cdot \max\left(e^{mt_\ell^* - i\eps} - e^{\eps_g}, 0\right)\right].$$
\end{lemma}
To apply Lemma~\ref{lem:apq_composition} with a fixed $\delta$, we use it to compute the largest $\eps$, at a granularity of 0.01, that achieves $(\eps_g,\delta)$-DP with some $\delta \leq 10^{-6}$, and we use this value for our experiments. As this is independent of the actual mechanism in question, the time required for this computation is not included in the runtime values reported for \apq.

\paragraph{\csmooth:} We start with the precise statement of the $t$-smooth sensitivity of computing a quantile:
\begin{lemma}[Proposition 3.4~\cite{NRS07}]
\label{lem:smooth_q}
    Let $a$ and $b$ be client-provided left and right data endpoints. Let $X$ be a database of values $x_1 \leq \ldots \leq x_n$ in $[a,b]$, and for notational convenience define $x_i = a$ for $i < 1$ and $x_i = b$ for $i > n$. Let $x_{j^*}$ be the true value for quantile $q$ on $X$. Then the $t$-smooth sensitivity of computing $q$ on $X$ is
    $$\S_q^t(X) = \max_{m=0, \ldots, n} \left(e^{-tm} \cdot \max_{k = 0, \ldots, m+1} (x_{j^*+k} - x_{j^*+k-m-1})\right).$$ 
\end{lemma}
Looking at the two $\max$ operations, we can compute $\S_q^t(X)$ in time $O(n^2)$. \citet{NRS07} also provide a slightly more involved method for computing $\S_q^t(X)$ in time $O(n\log(n))$. We omit its details here but note that our implementation uses this $O(n\log(n))$ speedup for the fairest time comparison. Next, we specify the exact noise distribution used to generate additive noise in \csmooth:
\begin{lemma}[Proposition 3~\cite{BS19}]
\label{lem:concentrated}
    Define the Laplace Log-Normal distribution with shape parameter $\sigma > 0$, $\lln{\sigma}$, as the distribution for the random variable $Z = X \cdot e^{\sigma Y}$ where $X \sim \lap{1}$ and $Y \sim N(0,1)$. Let $f$ be a real-valued function and let $s, t > 0$. Then releasing
    $$f(X) + \frac{\S_f^t(X) \cdot Z}{s}$$
    satisfies $\tfrac{\eps^2}{2}$-CDP for $\eps = \tfrac{t}{\sigma} + e^{1.5\sigma^2}s$.
\end{lemma}

Once we fix the desired CDP privacy parameter $\tfrac{\eps^2}{2}$, to apply \cref{lem:concentrated} we must still select $t, s, \sigma > 0$. We follow the selection method given in Sections 3.1.1 and 7.1 of~\citet{BS19}, omitting most of the details. First, for each of a sequence of values for $t$, we set $s = e^{-1.5\sigma^2}(\eps - t/\sigma)$ and numerically solve for $\sigma$ as a root of the polynomial $\tfrac{5\eps}t\sigma^3 - 5\sigma^2 - 1 = 0$. Repeating this process for each $t$ provides a collection of $(t, s, \sigma)$ triples without touching the database $X$. Given these triples $(t, s, \sigma)$, we finally select one to minimize variance $\tfrac{2\S_f^t(X)^2}{e^{-5\sigma^2}(\eps - t/\sigma)^2}$.

We pause to note that this last minimization of variance repeatedly touches $X$ to compute $\S_f^t(X)$ for different $t$. As this is not differentially private, we executed this non-private selection process once using data drawn from the standard Gaussian $N(0,1)$ and used the resulting values for \csmooth~experiments on our datasets. In practice, after starting from a wide range for $t$ of 150 logarithmically spaced values between $10^{-10}$ and $10$, we found that the values selected for $t$ clustered in a narrow subinterval across both data drawn from $N(0,1)$ and data drawn from our other experiment distributions. We therefore view the distribution-specific selection of $t$ as contributing relatively little to the final error of \csmooth.

\begin{wrapfigure}{L}{220pt}
\begin{center}
\vspace{-20pt}
\includegraphics[scale=0.5]{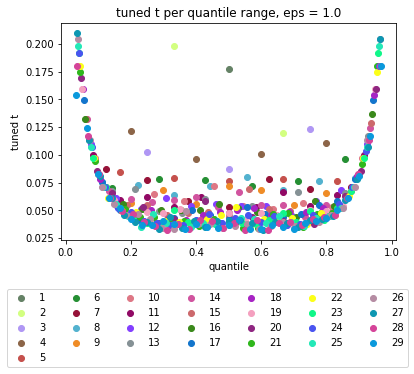}
\end{center}
\vspace{-15pt}
\caption{Tuned $t$ used for \csmooth~across different quantile ranges. For example, we used $t \approx 0.13$ for $m=3$, and $q = 0.5$ (the magenta dot in the middle of the plot).}
\end{wrapfigure}

In more detail, the actual $t$ selection process in our experiments is to use the variance-minimizing selection process described in \cref{subsec:algos} for each quantile in sets of quantiles ranging from $m=1$ to $m=29$ for $\eps = 1$. The range for $t$ is 50 logarithmically spaced values between 0.01 and 1. Each trial used 1000 samples drawn from $N(0,1)$ with data lower bound $-100$ and data upper bound $100$. Below, we record the $t$ selected for each quantile and quantile range, averaged across 5 trials. Each color represents a different set of quantiles, and each point for each color represents the $t$ selected for a single quantile.

\paragraph{\aggtree:} \aggtree 's hyperparameters are height $h$ and branching factor $b$. We tuned these parameters over the range $\{2, 3, \ldots, 15\}$ and $\{1, 2, \ldots, 15\}$ respectively. As with \csmooth, we used $N(0,1)$ data. The following two tables summarizes the values tuned over 50 trials of 1000 data points each.

\begin{table}[!ht]
    \centering
    \begin{tabular}{|*{16}{c|}}
    \hline
        \# quantiles & 1 & 2 & 3 & 4 & 5 & 6 & 7 & 8 & 9 & 10 & 11 & 12 & 13 & 14 & 15 \\
    \hline
	    height & 4 & 3 & 3 & 3 & 2 & 3 & 3 & 3 & 3 & 3 & 3 & 3 & 3 & 3 & 3 \\
    \hline
	    branching parameter & 4 & 6 & 6 & 9 & 14 & 10 & 7 & 7 & 10 & 10 & 8 & 7 & 7 & 12 & 10 \\
    \hline
    \end{tabular}
\end{table}

\begin{table}[!ht]
    \centering
    \begin{tabular}{|*{15}{c|}}
    \hline
        \# quantiles &16 & 17 & 18 & 19 & 20 & 21 & 22 & 23 & 24 & 25 & 26 & 27 & 28 & 29 \\
    \hline
	    height & 3 & 3 & 3 & 3 & 3 & 3 & 3 & 3 & 3 & 3 & 3 & 3 & 3 & 3 \\
    \hline
	    branching parameter & 10 & 10 & 10 & 7 & 10 & 10 & 7 & 10 & 12 & 12 & 12 & 10 & 10 & 12 \\
    \hline
    \end{tabular}
    \caption{Tuned height and branching parameters across number of quantiles.}
\end{table}

\section{Distance Error Experiments}
\label{sec:app_distance}
We conclude with experiments using a distance metric, which computes error as the average $\ell_1$ distance between the vectors of estimated and true quantiles: given quantile estimates $\hat o_1, \ldots, \hat o_m$ and true values $o_1, \ldots, o_m$, the error is $\|\hat o - o\|_1/m$. In this setting, we re-tune \aggtree 's hyperparameters using the distance metric, although the results are essentially the same:

\begin{table}[!ht]
    \centering
    \begin{tabular}{|*{16}{c|}}
    \hline
        \# quantiles & 1 & 2 & 3 & 4 & 5 & 6 & 7 & 8 & 9 & 10 & 11 & 12 & 13 & 14 & 15 \\
    \hline
	    height & 3 & 3 & 3 & 3 & 3 & 3 & 3 & 3 & 3 & 3 & 3 & 3 & 3 & 3 & 3 \\
    \hline
	    branching parameter & 8 & 6 & 6 & 10 & 9 & 9 & 10 & 10 & 12 & 12 & 10 & 10 & 7 & 10 & 8 \\
    \hline
    \end{tabular}
\end{table}

\begin{table}[!ht]
    \centering
    \begin{tabular}{|*{15}{c|}}
    \hline
        \# quantiles &16 & 17 & 18 & 19 & 20 & 21 & 22 & 23 & 24 & 25 & 26 & 27 & 28 & 29 \\
    \hline
	    height & 3 & 3 & 3 & 3 & 3 & 3 & 3 & 3 & 3 & 3 & 3 & 3 & 3 & 3 \\
    \hline
	    branching parameter & 8 & 10 & 7 & 10 & 10 & 12 & 8 & 10 & 12 & 10 & 7 & 5 & 7 & 5 \\
    \hline
    \end{tabular}
    \caption{Tuned height and branching parameters across number of quantiles (distance error).}
\end{table}

The final error plots appear below. Note that the algorithms that rely on the exponential mechanism (\apq~and \algo) at some point exhibit a sharp increase in error as $m$ grows. This is because these algorithms eventually end up sampling from a distribution that favors the extreme intervals containing the domain endpoints, and -- unlike misclassification error -- distance error strongly penalizes these outputs. Nonetheless, \algo~still achieves the strongest performance for a wide range of $m$.

\begin{figure}[!htbp]
    \begin{center}
        \includegraphics[scale=0.5]{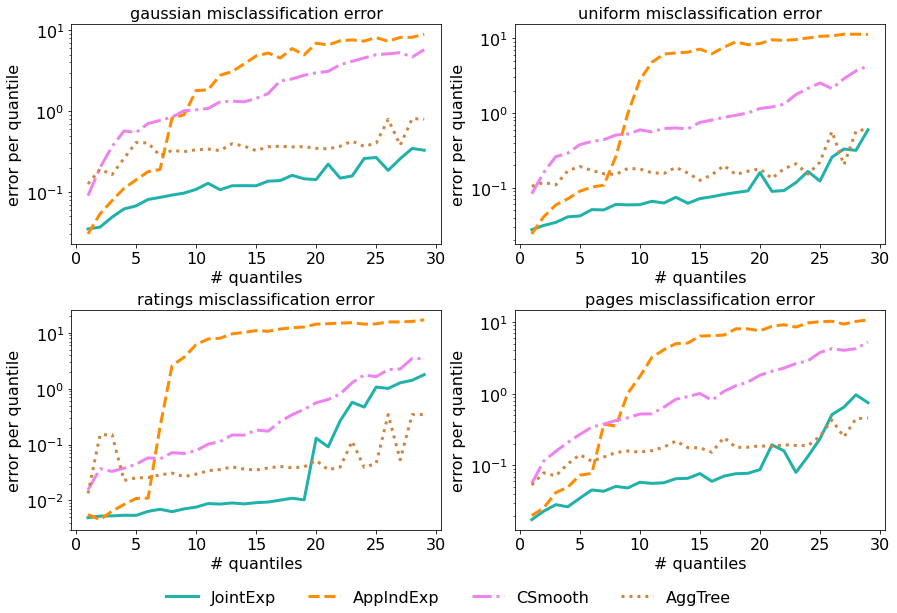}
    \end{center}
    \vspace{-15pt}
\caption{Distance error vs \# quantiles $m$ for $\eps=1$, averaged across 50 trials of 1,000 samples.}
\label{fig:time}
\end{figure}

\end{document}